\documentclass{article}

\usepackage[final,nonatbib]{nips_2017}

\usepackage{amsthm}
\usepackage{amsmath}
\usepackage[utf8]{inputenc} 
\usepackage[T1]{fontenc}    
\usepackage{url}            
\usepackage{booktabs}       
\usepackage{amsfonts}       
\usepackage{nicefrac}       
\usepackage{microtype}      
\usepackage{amssymb}
\usepackage{subfigure}
\usepackage{algorithm}  
\usepackage{algorithmic}  
\usepackage{graphicx}
\newtheorem{theorem}{Theorem} 
 
\newtheorem{corollary}{Corollary} 
\newtheorem{lemma}{Lemma} 
\newtheorem{proposition}{Proposition} 
\newtheorem{assumption}{Assumption}
\newtheorem{definition}{Definition} 
\title{Finite Sample Analysis of the GTD Policy Evaluation Algorithms in Markov Setting}

\newcommand{\E}{\mathbb{E}}

\newcommand{\huaf}{\mathcal{F}}
\newcommand\numberthis{\addtocounter{equation}{1}\tag{\theequation}}

\DeclareMathOperator*{\argmax}{argmax}

\graphicspath{{./figure/}}

\author{
	 Yue Wang 	\thanks{This work was done when the first author was visiting Microsoft Research Asia.} \\
	School of  Science\\
	Beijing Jiaotong University\\
	\texttt{11271012@bjtu.edu.cn} \\	
	\And
	Wei Chen \\
	Microsoft Research\\
	\texttt{wche@microsoft.com  } \\
	\And
	Yuting Liu \\
	School of  Science\\
	Beijing Jiaotong University \\
	\texttt{ytliu@bjtu.edu.cn  } \\
	\And
	Zhi-Ming Ma \\
	Academy of Mathematics and Systems Science\\
	Chinese Academy of Sciences \\
	\texttt{mazm@amt.ac.cn  } \\  
	\And
	Tie-Yan Liu \\
	Microsoft Research\\
	\texttt{Tie-Yan.Liu@microsoft.com } \\   	
}

\begin{document}
	
	\maketitle
	
	\begin{abstract}
		
		In reinforcement learning (RL) , one of the key components is policy evaluation, which aims to estimate the value function (i.e., expected long-term accumulated reward) of a policy. With a good policy evaluation method, the RL algorithms will estimate the value function more accurately and find a better policy. When the state space is large or continuous \emph{Gradient-based Temporal Difference(GTD)} policy evaluation algorithms with linear function approximation are widely used. Considering that the collection of the evaluation data is both time and reward consuming, a clear understanding of the finite sample performance of the policy evaluation algorithms is very important to reinforcement learning. Under the assumption that data are i.i.d. generated, previous work provided the finite sample analysis of the GTD algorithms with constant step size by converting  them into  convex-concave saddle point problems. However, it is well-known that, the data are generated from Markov processes rather than i.i.d. in RL problems.. In this paper, in the realistic Markov setting, we derive the finite sample bounds for the general convex-concave saddle point problems, and hence for the GTD algorithms. We have the following discussions based on our bounds. (1) With variants of step size, GTD algorithms converge. (2) The convergence rate is determined by the step size, with the mixing time of the Markov process as the coefficient. The faster the Markov processes mix, the faster the convergence. (3) We explain that the experience replay trick is effective by improving the mixing property of the Markov process.  To the best of our knowledge, our analysis is the first to provide finite sample bounds for the GTD algorithms in Markov setting.

	\end{abstract}

	\section{Introduction}

	Reinforcement Learning (RL) (\cite{sutton1998reinforcement}) technologies are very powerful to learn how to interact with environments, and has variants of important applications, such as robotics, computer games and so on (\cite{kober2013reinforcement}, \cite{mnih2015human}, \cite{silver2016mastering}, \cite{bahdanau2016actor}). 
	
	In RL problem, an agent observes the current state, takes an action following a policy at the current state, receives a reward from the environment, and the environment transits to the next state in Markov, and again repeats these steps. The goal of the RL algorithms is to find the optimal policy which leads to the maximum long-term reward. The value function of a fixed policy for a  state is defined as the expected long-term accumulated reward the agent would receive by following the fixed policy starting from this state. Policy evaluation aims to accurately estimate the value of all states under a given policy, which is a key component in RL (\cite{sutton1998reinforcement}, \cite{dann2014policy}). A better policy evaluation method will help us to better improve the current policy and find the optimal policy.
	
	When the state space is large or continuous, it is inefficient to represent the value function over all the states by a look-up table. A common approach is to extract features for states and use parameterized function over the feature space to approximate the value function. In applications, there are linear approximation and non-linear approximation (e.g. neural networks) to the value function. In this paper, we will focus on the linear approximation (\cite{sutton2009convergent},\cite{sutton2009fast},\cite{liu2015finite}). Leveraging the localization technique in \cite{bhatnagar2009convergent}, the results can be generated into non-linear cases with extra efforts. We leave it as future work. 
	
	In policy evaluation with linear approximation, there were substantial work for the temporal-difference (TD) method, which uses the Bellman equation to update the value function during the learning process (\cite{sutton1988learning},\cite{tsitsiklis1997analysis}). Recently, \cite{sutton2009convergent} \cite{sutton2009fast} have proposed \emph{Gradient-based Temporal Difference (GTD)} algorithms which use gradient information of the error from the Bellman equation to update the value function.  It is shown that, GTD algorithms have achieved the lower-bound of the storage and computation complexity, making them powerful to handle high dimensional big data.  Therefore, now GTD algorithms are widely used in policy evaluation problems and the policy evaluation step in practical RL algorithms (\cite{bhatnagar2009convergent},\cite{silver2014deterministic}).  

	However, we don’t have sufficient theory to tell us about the finite sample performance of the GTD algorithms. To be specific, will the evaluation process converge with the increasing of the number of the samples? If yes, how many samples we need to get a target evaluation error? Will the step size in GTD algorithms influence the finite sample error? How to explain the effectiveness of the practical tricks, such as experience replay? Considering that the collection of the evaluation data is very likely to be both time and reward consuming, to get a clear understanding of the finite sample performance of the GTD algorithms is very important to the efficiency of policy evaluation and the entire RL algorithms.
	
	Previous work (\cite{liu2015finite}) converted the objective function of GTD algorithms into a convex-concave saddle problem and conducted the finite sample analysis for GTD with constant step size under the assumption that data are i.i.d. generated. However, in RL problem, the date are generated from an agent who interacts with the environment step by step, and the state will transit in Markov as introduced previously. As a result, the data are generated from a Markov process but not i.i.d..  In addition, the work did not study the decreasing step size, which are also commonly-used in many gradient based algorithms(\cite{sutton2009convergent},\cite{sutton2009fast},\cite{yu2015convergence}). Thus, the results from previous work cannot provide satisfactory answers to the above questions for the finite sample performance of the GTD algorithms.
	
	In this paper, we perform the finite sample analysis for the GTD algorithms in the more realistic Markov setting. To achieve the goal, first of all, same with \cite{liu2015finite}, we consider the stochastic gradient descent algorithms of the general convex-concave saddle point problems, which include the GTD algorithms. The optimality of the solution is measured by the primal-dual gap (\cite{liu2015finite}, \cite{nemirovski2009robust}). The finite sample analysis for convex-concave optimization in Markov setting is challenging. On one hand, in Markov setting, the non-i.i.d. sampled gradients are no longer unbiased estimation of the gradients. Thus, the proof technique for the convergence of convex-concave problem in i.i.d. setting cannot be applied. On the other hand, although SGD converge for convex optimization problem with the Markov gradients, it is much more difficult to obtain the same results in the more complex convex-concave optimization problem.

	To overcome the challenge, we design a novel decomposition of the error function (i.e. Eqn (\ref{keydecomposition})). The intuition of the decomposition and key techniques are as follows:  {(1)} Although samples are not i.i.d., for large enough $ \tau $, the sample at time $t+\tau$ is "nearly independent"  of the sample  at time $t$, and its distribution is "very close" to the stationary distribution.  {(2)} We split the random variables in the objective related to $ \mathbf{\E} $ operator and the variables related to $ \mathbf{\max} $ operator into different terms in order to control them respectively. It is non-trivial, and we construct a sequence of auxiliary random variables to do so.  {(3)} All constructions above need to be carefully considered the measurable issues in the Markov setting. {(4)} We construct new martingale difference sequences and apply Azuma's inequality to derive the high-probability bound from the in-expectation bound.
	
	By using the above techniques, we prove a novel finite sample bound for the convex-concave saddle point problem. Considering the GTD algorithms are specific convex-concave saddle point optimization methods, we finally obtained the finite sample bounds for the GTD algorithms, in the realistic Markov setting for RL. To the best of our knowledge, our analysis is the first to provide finite sample bounds for the GTD algorithms in Markov setting.

	We have the following discussions based on our finite sample bounds. 
	\begin{enumerate}
		\item	  GTD algorithms do converge, under a flexible condition on the step size, i.e.  $\sum_{t=1}^{T}\alpha_t \to \infty, \frac{\sum_{t=1}^{T}\alpha_t^2}{\sum_{t=1}^{T}\alpha_t} <\infty $, as $T\to \infty$, where $\alpha_t$ is the step size. Most of step sizes used in practice satisfy this condition. 
		\item  The convergence rate is  \small $ \mathcal{O}\left(\sqrt{(1+\tau(\eta))\frac{\sum_{t=1}^{T}\alpha_t^2}{\sum_{t=1}^{T}\alpha_t } + \frac{\sqrt{\tau(\eta)\log(\frac{\tau(\eta)}{\delta}) \sum_{t=1}^{T}\alpha_t^2}} {\sum_{t=1}^{T}\alpha_t}}  \right)  $\normalsize, where $\tau(\eta)$   is the mixing time of the Markov process, and $\eta$ is a constant. Different step sizes  will lead to different convergence rates.
		\item   The experience replay trick is effective, since it can improve the mixing property of the Markov process.  
	\end{enumerate}
	
	Finally, we conduct simulation experiments to verify our theoretical finding. All the conclusions from the analysis are consistent with our empirical observations. 
	
	\section{Preliminaries}
	In this section, we briefly introduce the GTD algorithms and related works.
	\subsection{Gradient-based TD algorithms}
	Consider the reinforcement learning problem with Markov decision process(MDP) $ (\mathcal{S},\mathcal{A},P,R,\gamma) $, where $ \mathcal{S} $ is the state space, $ \mathcal{A} $ is the action space, $ P= \{P_{s,s'}^a; s,s’\in \mathcal{S}, a\in\mathcal{A}\} $ is the transition matrix and $ P_{s,s'}^a $ is the transition probability from state $ s $ to state $ s' $ after taking action $ a $, $ R=\{R(s,a );s\in \mathcal{S},a\in\mathcal{A}$ is the reward function and $R(s,a)$ is the reward received at state $s$ if taking action $a$, and $ 0<\gamma<1 $ is the discount factor. 	
	A policy function $ \mu: \mathcal{A}\times\mathcal{S}\to [0,1]$ indicates the probability to take each action at each state.  Value function for policy $ \mu $ is defined as: $ 	V^{\mu}(s)\triangleq E\left[ \sum_{t=0}^{\infty}\gamma^t R(s_t,a_t)|s_0 = s,\mu  \right] $.

	In order to perform policy evaluation in a large state space, states are represented by a feature vector $\phi(s)\in\mathbb{R}^d$, and a linear   function  $ \hat{v}(s) = \phi(s)^\top \theta  $
	is used to approximate the value function. The evaluation error is defined as  $ \Vert V(s) - \hat{v}(s) \Vert_{s\sim \pi} $, which can be decomposed into approximation error and estimation error. In this paper, we will focus on the estimation error with linear function approximation. 
	
	As we know, the value function in RL satisfies the following Bellman equation: $ V^\mu(s)   = \E_{\mu,P}\left[R(s_t,a_t)+\gamma V^\mu(s_{t+1}) | s_t=s \right] \triangleq T^\mu V^\mu(s)  $,
	where $ T^\mu  $ is called Bellman operator for policy $\mu$.
	Gradient-based TD (GTD) algorithms (including GTD and GTD2) proposed by \cite{sutton2009convergent} and \cite{sutton2009fast}  update the approximated value function by minimizing the objective function related to  Bellman equation errors, i.e., the  norm of the expected TD update (NEU) and mean-square projected Bellman error  (MSPBE) respectively(\cite{maei2011gradient},\cite{liu2015finite})	,
	\begin{align}
	GTD&:  \quad J_{NEU}(\theta)  =\Vert \Phi^\top K (T^\mu\hat{v}-\hat{v}) \Vert^2\\
	GTD2&: \quad J_{MSPBE}(\theta)  = \Vert \hat{v} - \mathcal{P} T^\mu\hat{v}\Vert = \Vert \Phi^\top K (T^\mu\hat{v}-\hat{v}) \Vert^2_{C^{-1}}
	\end{align}
	where $ K $ is a   diagonal matrix whose elements are $ \pi(s) $, $ C  = \E_\pi(\phi_i\phi_i^\top ) $,  and $ \pi $ is a distribution over the state space $ \mathcal{S} $.
	
	Actually, the two objective functions in GTD and GTD2 can be unified as below 
	\begin{align}\label{rl objective function}
	J(\theta) =  \Vert b-A\theta \Vert_{M^{-1}}^2,
	\end{align} 	
	where $M=I$ in GTD, $ M = C$, in GTD2,  $  A=\E_\pi[\rho(s,a)  \phi(s)(\phi(s)-\gamma\phi(s'))^\top],  b=\E_\pi[\rho(s,a) \phi(s) r], \rho(s,a) = \mu(a |s ) / \mu_b(a |s )$ is  the importance weighting factor. Since the underlying distribution is unknown, we use the data $ \mathcal{D} = \left \lbrace \xi_i= (s_i,a_i,r_i,s_i') \right\rbrace_{i=1}^n $ to estimate the value function by minimizing the empirical estimation error, i.e., 
	\small$$ \hat{J}(\theta) =  1/T\sum_{i=1}^T\Vert \hat{b}-\hat{A}\theta \Vert_{\hat{M}^{-1}}^2$$\normalsize   
	where $ \hat{A}_i= \rho(s_i,a_i) \phi(s_i)(\phi(s_i)-\gamma\phi(s_i'))^\top,  \hat{b}_i= \rho(s_i,a_i)\phi(s_i) r_i, \hat{C}_i = \phi(s_i)\phi(s_i) ^\top $.		
	
	\cite{liu2015finite} derived that the GTD algorithms to minimize     (\ref{rl objective function}) is equivalent to the stochastic gradient algorithms to solve the following convex-concave saddle point problem
	\small \begin{align}\label{gtd minmax form}
	\min_x\max_y\left( L(x,y) = \langle b-Ax,y \rangle - \frac{1}{2}\Vert y\Vert^2_M \right),
	\end{align}\normalsize
	with $ x $ as the  parameter $ \theta $ in the value function, $ y $ as the auxiliary variable used in GTD algorithms.
	\small\begin{algorithm}[tb] \label{gtd algorithms}			
		\caption{ GTD Algorithms}
		\begin{algorithmic}[1]\small
			\FOR{$ t = 1, \dots, T $}
			\STATE  
			$ \text{Update parameters:   }\quad 	 y_{t+1} = \mathcal{P}_{\mathcal{X}_y}\left(y_t + \alpha_t(\hat{b}_t - \hat{A}_t\theta_t -\hat{M}_ty_t)\right) \quad x_{t+1} = \mathcal{P}_{\mathcal{X}_x}\left(x_t + \alpha_t\hat{A}_t^\top y_t\right) $
			\ENDFOR
			\ENSURE 				
			$ \quad 	 \tilde{x}_T = \frac{\sum_{t=1}^{T}\alpha_t x_t}{\sum_{t=1}^{T}\alpha_t}  \qquad \tilde{y}_T = \frac{\sum_{t=1}^{T}\alpha_t y_t}{\sum_{t=1}^{T}\alpha_t} $				
		\end{algorithmic}
	\end{algorithm}	\normalsize
	Therefore, we consider the general convex-concave stochastic saddle point problem as below 
	\begin{align}\label{saddle point problem}
	\min_{x\in\mathcal{X}_x}\max_{y\in\mathcal{X}_y}\lbrace \phi(x,y) = \E_\xi[\Phi(x,y,\xi)]\rbrace,
	\end{align}
	where $ \mathcal{X}_x \subset \mathbb{R}^n $ and $ \mathcal{X}_y \subset \mathbb{R}^m$ are bounded closed convex sets, $ \xi \in \Xi $ is random variable and its distribution is $ \Pi(\xi) $, and the expected function $ \phi(x,y) $ is convex in $ x $ and concave in $ {y } $. 	Denote  $    z = (x,y) \in \mathcal{X}_x\times \mathcal{X}_y \triangleq \mathcal{X}$, the gradient  of $\phi(z)$ as $ g(z) $, and the gradient of  $ \Phi(z,\xi) $ as $ G(z,\xi) $.
	
	In the stochastic gradient algorithm, the model is updated as: $ 	z_{t+1} =	\mathcal{P}_{\mathcal{X} }( z_t - \alpha_t(G(z_t,\xi_t))) $, where  $ \mathcal{P}_{\mathcal{X}}  $ is the projection onto $\mathcal{X}$ and $\alpha_t$ is the step size.
	After $T$ iterations, we get the model 
	$ 	\tilde{z}_1^T = \frac{\sum_{t=1}^{T}\alpha_tz_t }{\sum_{t=1}^{T}\alpha_t } $.	
	The error of the model $ \tilde{z}_1^T $ is measured by the primal-dual gap error 
	\small\begin{equation}\label{error function}
	Err_\phi(\tilde{z}_1^T) = \max_{y\in \mathcal{X}_y}\phi( \tilde{x}_1^T,y ) - \min_{x\in\mathcal{X}_x}\phi(x,\tilde{y}_1^T).
	\end{equation} \normalsize	
	\cite{liu2015finite} proved that the estimation error of the GTD algorithms can be upper bounded by their corresponding primal-dual gap error multiply a factor. Therefore, we are going to derive the finite sample primal-dual gap error bound for the  convex-concave saddle point problem firstly, and then extend it to the finite sample estimation error bound for the GTD algorithms.
	
	Details of  GTD algorithms used to optimize (\ref{gtd minmax form}) are placed in \textbf{Algorithm 1}( \cite{liu2015finite}).
	
	\subsection{Related work}
	The TD algorithms for policy evaluation can be divided into two categories: gradient based methods and least-square(LS) based methods(\cite{dann2014policy}). Since LS based algorithms need $ \mathcal{O}(d^2) $ storage and computational complexity while GTD algorithms are both of $ \mathcal{O}(d) $ complexity, gradient based algorithms are more commonly used when the feature dimension is large. Thus, in this paper, we focus on GTD algorithms.
	
	\cite{sutton2009convergent}  proposed the gradient-based temporal difference (GTD) algorithm for off-policy policy  evaluation problem with linear function approximation.  \cite{sutton2009fast} proposed GTD2 algorithm which shows a faster convergence in practice. \cite{liu2015finite}  connected GTD algorithms to a convex-concave saddle point problem and derive a finite sample bound in both on-policy and off-policy cases for constant step size in i.i.d.  setting. 
	
	In the realistic Markov setting, although the finite sample bounds for LS-based algorithms have been proved (\cite{lazaric2012finite} \cite{tagorti2015rate}) LSTD($ \lambda $), to the best of our knowledge, there is no previous  finite sample analysis work for GTD algorithms. 
	
	\section{Main Theorems}
	
	In this section, we will present our main results. In    Theorem \ref{highprobability}, we present our finite sample bound for the general convex-concave saddle point problem; in Theorem \ref{gtdbound}, we provide the finite sample bounds for GTD algorithms in both on-policy and off-policy cases.  Please refer the complete proofs in the supplementary materials.

	Our results are derived based on the following common assumptions(\cite{nemirovski2004prox}, \cite{duchi2012ergodic}, \cite{liu2015finite}).
	Please note that, the bounded-data property in assumption 4 in RL can guarantee the Lipschitz and smooth properties in assumption \ref{assumptionlipschitz}-\ref{assumptionsmooth} (Please see Propsition \ref{GTD bound prepare} ). 
	
	\begin{assumption}[Bounded parameter]\label{assumptioncompactness}
		There exists $ D>0 $, such that	$ \Vert z-z'\Vert \le D, ~~for~ \forall z,z'  \in \mathcal{X} $.
	\end{assumption}

	\begin{assumption}[Step size]\label{assumptionstepsize}
		The step size $ \alpha_t $ is non-increasing.
	\end{assumption}
	
	\begin{assumption}[Problem solvable]\label{assrl nonsingular}
		The matrix $ A  $ and $ C $ in Problem \ref{gtd minmax form} are non-singular.
	\end{assumption} 	
	
	\begin{assumption}[Bounded data]\label{assrl bound}
		Features are bounded  by $ L $, rewards are bounded by $ R_{max} $  and importance weights are bounded by $ \rho_{max} $.
	\end{assumption}
	\begin{assumption}[Lipschitz]\label{assumptionlipschitz}
		For $ \Pi $-almost every $ \xi $, the function $ \Phi(x,y,\xi) $ is Lipschitz   for both x and y, with finite constant $ L_{1x }, L_{1y}   $, respectively.  We Denote $ L_1 \triangleq \sqrt{2} \sqrt{L_{1x}^2 + L_{1y}^2} $.	
	\end{assumption}		
	\begin{assumption}[Smooth]\label{assumptionsmooth}			
		For $ \Pi $-almost every $ \xi $, the partial  gradient function  of $ \Phi(x,y,\xi) $ is Lipschitz  for both x and y  with finite constant $ L_{2x}, L_{2y}  $ respectively. We denote $ L_2 \triangleq  \sqrt{2}\sqrt{L_{2x}^2 + L_{2y}^2} $.			
	\end{assumption}
	
	For Markov process, the mixing time characterizes how fast the process converge to its stationary distribution.	Following the notation of \cite{duchi2012ergodic}, we denote the conditional probability distribution $ P(\xi_t \in A | \huaf_s) $ as $ P_{[s]}^t(A) $ and the corresponding probability density as $ p_{[s]}^t $. Similarly, we denote the stationary distribution of the data generating stochastic process as $ \Pi $ and its density  as $ \pi  $.

	\begin{definition}\label{def mixing time}
		The mixing time $ \tau(P_{[t]},\eta)$ of the sampling distribution P conditioned on the $ \sigma-$field of the initial t sample $ \huaf_t = \sigma(\xi_1,\dots,\xi_t) $ is defined as: $ \tau(P_{[t]},\eta)\triangleq \inf\left\lbrace \Delta  : t\in \mathbb{N}, \int |p_{[t]}^{t+\Delta}(\xi) - \pi(\xi)|d(\xi) \le \eta\right\rbrace $, where $ p_{[t]}^{t+\Delta} $ is the conditional probability density at time $t + \Delta $, given $\mathcal{F}_t$. 	
	\end{definition}

	\begin{assumption}[Mixing time]\label{assume mixing time}
		The mixing times of the stochastic process $ \lbrace\xi_t \rbrace$  are uniform. i.e., there exists uniform mixing times $\tau(P,\eta) \le \infty$ such that, with probability $1$, we have $ \tau(P_{[s]},\eta)\le \tau(P, \eta) $ for all $ \eta >0 $ and $ s \in \mathbb{N} $. 
	\end{assumption}

	Please note that, any time-homogeneous Markov chain with finite state-space and any uniformly ergodic Markov chains with general state space satisfy the above assumption(\cite{meyn2012markov}). For simplicity and without of confusion, we will denote $ \tau(P, \eta) $ as $ \tau(\eta) $.

	\subsection{Finite Sample Bound for Convex-concave Saddle Point Problem}
	
	\begin{theorem}\label{highprobability}
		
		Consider the convex-concave problem in Eqn (2.5). Suppose Assumption \ref{assumptioncompactness},\ref{assumptionstepsize},\ref{assumptionlipschitz},\ref{assumptionsmooth} hold. Then for the gradient algorithm optimizing the convex-concave saddle point problem in (\ref{saddle point problem}), for $\forall \delta>0$ and $\forall \eta > 0 $ such that $ \tau(\eta)\le T/2$, with probability at least $ 1-\delta$, we have
		\small
		\begin{align*}
		Err_\phi(\tilde{z}_1^T)   
		\le  		\frac{1}{\sum\limits_{t=1}^{T}\alpha_t} \Bigg[ A  &+B \sum_{t=1}^{T}\alpha_t^2 + C\tau(\eta)\sum_{t=1}^{T}\alpha_t^2  + F\eta\sum_{t=1}^{T}\alpha_t + H\tau(\eta) \\
		&+  8DL_1 \sqrt{2\tau(\eta) \log\frac{\tau(\eta)}{\delta} \left(\sum_{t=1}^{T}\alpha_t^2 + \tau(\eta)\alpha_0 \right)} \Bigg]
		\end{align*}
		\normalsize
		\begin{small}
			\begin{align*}
			\text{where : }
			A = D^2   \qquad
			B = \frac{5}{2} L_1^2 \qquad
			C = 6L_1^2 + 2L_1L_2D \qquad
			F = 2L_1D  \qquad
			H = 6L_1D\alpha_0  
			\end{align*} 
			
		\end{small}

	\end{theorem}

	\begin{proof}[Proof Sketch of Theorem \ref{highprobability}]		
		
		By the definition of the error function in (\ref{error function}) and the property that $ \phi(x,y) $ is convex for $ x $ and concave for $ y $, the expected error can be bounded as below
		\small
		$$        Err_\phi(\tilde{z}_1^T)   \le     \max_z \frac{1}{\sum_{t=1}^{T}\alpha_t}\sum_{t=1}^{T}\alpha_t \left[(z_t-z)^\top g(z_t) \right].$$\normalsize
		Denote $ \delta_t \triangleq g(z_t) - G(z_t,\xi_{t}) $, $ \delta'_t \triangleq g (z_t) - G( z_t,\xi_{t+\tau} ) $,  $ \delta''_t \triangleq G (z_t,\xi_{t+\tau}) -  G (z_t,\xi_{t})  $. Constructing  $ \lbrace v_t \rbrace_{t\ge 1} $  which is measurable with respect to $ \huaf_{t-1} $,$v_{t+1} = P_{\mathcal{X}}\big(v_t - \alpha_t( g (z_t) - G(z_t,\xi_{t})) \big) $. We have the following key decomposition to the right hand side in the above inequality,  the initiation and the explanation for such decomposition is placed in supplementary materials. For $ \forall \tau \ge 0$:	 
		
		\small\begin{align*}\numberthis\label{keydecomposition} 
		\max_z \sum_{t=1}^{T}\alpha_t \left[(z_t-z)^\top g(z_t) \right]   & =	    \max_z \Biggl[ \sum\limits_{t=1}^{T-\tau}\alpha_t \bigg[\underbrace{(z_t-z)^\top  G (z_t,\xi_{t})}_{(a)} +\underbrace{(z_t-v_t)^\top \delta'_t}_{(b)}  \\
		& + \underbrace{(z_t-v_t)^\top \delta''_t}_{(c)} + \underbrace{(v_t-z)^\top \delta_t }_{(d)}    \bigg]
		+  \underbrace{ \sum_{t=T-\tau + 1}^{T}\alpha_t \left[(z_t-z)^\top g(z_t)\right]}_{(e)}\Biggr].
		\end{align*} 	\normalsize

		For term(a), we split $ G(z_t,\xi_t) $ into three terms by the definition of $ \mathcal{L} _2$-norm and the iteration formula of $ z_t $, and then we bound its summation by $ \sum_{t=1}^{T-\tau} \left(\Vert \alpha_t G(z_t,\xi_t) \Vert^2  + \Vert z_t - z \Vert ^2 - \Vert z_{t+1}  - z\Vert^2 \right) $. Actually, in the summation, the last two terms will be eliminated except for their first and the last terms. Swap the $ \max $ and $ \sum $ operators and use the Lipschitz Assumption \ref{assumptionlipschitz}, the first term can   be bounded.  Term (c) includes the sum of $  G (z_t,\xi_{t+\tau}) -  G (z_t,\xi_{t})  $, which is might be large in Markov setting. We reformulate it into the sum of $  G (z_{t-\tau},\xi_{t }) -  G (z_t,\xi_{t}) $  and use the smooth Assumption \ref{assumptionsmooth} to bound it. Term (d) is similar to term (a) except that $ g(z_t) - G(z_t,\xi_t) $ is the gradient that used to update $ v_t $. We can bound it similarly with term (a). Term(e) is a constant that does not change much with $ T\to \infty $, and we can bound it directly through upper bound of each of its own terms. Finally, we combine all the upper bounds to each term,  use the mixing time Assumption \ref{assume mixing time} to choose $ \tau = \tau(\eta) $ and  obtain the error bound in Theorem 1.			
		
		We decompose Term(b) into a martingale part and an expectation part.By  constructing a martingale difference sequence and using the Azuma's inequality together with the Assumption \ref{assume mixing time}, we can bound Term (b) and finally obtain the high probability error bound.
	\end{proof}	
	\textbf{Remark:} (1) With $T\to\infty$, the error bound approaches $0$ in order $O(\frac{\sum_{t=1}^T \alpha_t^2}{\sum_{t=1}^T \alpha_t})$. (2) The mixing time $\tau(\eta)$ will influence the convergence rate. If the Markov process has better mixing property with smaller $\tau(\eta)$, the algorithm converge faster. (3) If the data are i.i.d. generated (the mixing time $ \tau(\eta) =0, \forall \eta$) and the step size is set to the constant $\frac{c}{L_1 \sqrt{T}} $, our  bound will reduce to  $  Err_\phi(\tilde{z}_1^T) 
	\le 		\frac{1}{\sum_{t=1}^{T}\alpha_t} \left[ A    + B \sum_{t=1}^{T}\alpha_t^2  \right]= \mathcal{O}(\frac{L_1 }{\sqrt{T}})$, which is identical to previous work with constant step size in i.i.d. setting (\cite{liu2015finite},\cite{nemirovski2009robust}). (4)	  The  high probability bound is similar to the expectation bound in the following Lemma \ref{expectationbound} except for the last term. This is because we consider the deviation of the data around its expectation  to derive the high probability bound.

	\begin{lemma}\label{expectationbound}
		
		Consider the convex-concave problem (2.5), under the same as Theorem \ref{highprobability}, we have
		\small\begin{align*}
		\E_{\mathcal D} [Err_\phi(\tilde{z}_1^T)] 
		\le 		\frac{1}{\sum_{t=1}^{T}\alpha_t} \left[ A  +B \sum_{t=1}^{T}\alpha_t^2 + C\tau(\eta)\sum_{t=1}^{T}\alpha_t^2  + F\eta\sum_{t=1}^{T}\alpha_t + H\tau(\eta) \right], \forall \eta>0,
		\end{align*}\normalsize			
		
	\end{lemma}
	
	\begin{proof}[Proof Sketch of Lemma \ref{expectationbound}]				
		We start from  the key decomposition (\ref{keydecomposition}), and bound each term with expectation this time. We can easily bound each term as previously except for Term (b). For term (b), since   $(z_t-v_t) $ is not related to $ \max  $ operator and it is measurable with respect to $ \huaf_{t-1} $, we can bound Term (b) through the definition of mixing time and finally obtain the expectation bound.
	\end{proof}

	\subsection{Finite Sample Bounds for GTD Algorithms}
	As a specific convex-concave saddle point problem,  the error bounds in Theorem 1\&2 can also provide the error bounds for GTD with the following specifications for the Lipschitz constants.
	
	\begin{proposition}\label{GTD bound prepare}
		Suppose Assumption \ref{assumptioncompactness}-\ref{assrl bound}  hold, then the objective function in GTD algorithms is Lipschitz and smooth with the following coefficients:			
			\begin{align*}
		L_1 &\le \sqrt{2}( 2D(1+\gamma)\rho_{max} L^2 d + \rho_{max}LR_{max} +\lambda_M) \\
		L_2 &\le \sqrt{2}(2(1+\gamma)\rho_{max} L^2 d + \lambda_M)
		\end{align*}
						where $\lambda_M$ is the largest singular value of $M$.
	\end{proposition}

	\begin{theorem}\label{gtdbound} 
		Suppose assumptions 1-4 hold, then we have the following finite sample bounds for the   error $ \Vert V- \tilde{v}_1^T \Vert_\pi $ in GTD algorithms: In on-policy case, the  bound  in expectation is \small$\mathcal{O} \left(\frac{L \sqrt{L^4 d^3 \lambda_M\pi_{max} (1+\tau(\eta))\pi_{max}o_1(T)}}{\nu_C}\right)$\normalsize and with probability $ 1-\delta $ is \small$\mathcal{O}\left(\frac{\sqrt{  L^4 d^2 \lambda_M\pi_{max} }}{\nu_C} \left( \sqrt{(1+\tau(\eta))L^2 d o_1(T) + \sqrt{\tau(\eta)\log\left(\frac{\tau(\eta)}{\delta} \right)}o_2(T)} \right) \right)$\normalsize;		
		In off-policy case, the  bound in expectation is \small$\mathcal{O} \left(\frac{L^2d\sqrt{2\lambda_C\lambda_M\pi_{max} (1+\tau(\eta))o_1(T)}}{\nu_{(A^TM^{-1}A)}}\right)$\normalsize and with probability $ 1-\delta $ is \small$\mathcal{O} \left(\frac{\sqrt{2\lambda_C\lambda_M\pi_{max} }}{\nu_{(A^TM^{-1}A)}}  \left( \sqrt{L^4d^2(1+\tau(\eta))o_1(T) + \sqrt{\tau(\eta)\log{(\frac{\tau(\eta)}{\delta})}}o_2(T) } \right) \right)$\normalsize, 
		where $ \nu_C, \nu_{(A^TM^{-1}A)}$ is the smallest eigenvalue of the $C$ and $ A^TM^{-1}A $ respectively, $\lambda_C$ is the largest singular value of $C$, $ o_1(T) =   (\frac{\sum_{t=1}^{T}\alpha_t^2}{\sum_{t=1}^{T}\alpha_t}), o_2(T) =   (\frac{\sqrt{\sum_{t=1}^{T}\alpha_t^2}}{\sum_{t=1}^{T}\alpha_t}) $. 		
		
	\end{theorem}
	
	We would like to make the following discussions for   Theorem \ref{gtdbound}.	
	
	\textbf{The GTD algorithms do converge in the realistic Markov setting. }As  in Theorem \ref{gtdbound}, the  bound  in expectation is   $ \mathcal{O}\left( \sqrt{(1+\tau(\eta))o_1(T)} \right) $ and with   probability $ 1-\delta $ is \small $ \mathcal{O}\left( \sqrt{(1+\tau(\eta))o_1(T) + \sqrt{\tau(\eta)\log(\frac{\tau(\eta)}{\delta}) } o_2(T)}   \right)$\normalsize. If the step size $\alpha_t$ makes $o_1(T)\to 0$ and $o_2(T)\to 0$, as $T\to \infty$, the GTD algorithms will converge. Additionally, in high probability bound, if $ \sum_{t=1}^{T}\alpha_t^2 > 1 $, then $ o_1(T) $   dominates the order, if $ \sum_{t=1}^{T}\alpha_t^2 < 1 $, $o_2(T)$ dominates.
	
	\textbf{The setup of the step size can be flexible.} Our finite sample bounds for GTD algorithms converge to $0$ if the step size satisfies $\sum_{t=1}^{T}\alpha_t \to \infty, \frac{\sum_{t=1}^{T}\alpha_t^2}{\sum_{t=1}^{T}\alpha_t} <\infty $, as $T\to \infty$. This condition on step size is much weaker than the constant step size in previous work \cite{liu2015finite}, and the common-used step size $\alpha_t =  \mathcal{O}(\frac{1}{\sqrt{t}}), \alpha_t =\mathcal{O}(\frac{1}{t}),\alpha_t =c = \mathcal{O}(\frac{1}{\sqrt{T}} ) $  all satisfy the condition. To be specific, for $ \alpha_t = \mathcal{O}(\frac{1}{\sqrt{t}})  $, the convergence rate is $\mathcal{O}(\frac{\ln(T)}{\sqrt{T}})$; for $ \alpha_t = \mathcal{O}(\frac{1}{t}) $, the convergence rate is $ \mathcal{O}(\frac{1}{\ln(T)})$, for the constant step size, the optimal setup is $ \alpha_t =\mathcal{O}( \frac{1}{\sqrt{T}} ) $ considering the trade off between  $ o_1(T) $ and $ o_2(T) $, and the convergence rate is $ \mathcal{O}(\frac{1}{\sqrt{T}}) $.

	\textbf{The mixing time matters. } If the data are generated from a Markov process with smaller mixing time, the error bound will be smaller, and we just need fewer samples to achieve a fixed estimation error. This finding can explain why the experience replay trick (\cite{lin1993reinforcement}) works. With experience replay, we store the agent’s experiences (or data samples) at each step, and randomly sample one from the pool of stored samples to update the policy function.  By Theorem 1.19 - 1.23 of \cite{durrett2016poisson}, it can be proved that, for arbitrary $\eta>0$, there exists $t_0$, such that $ \forall t>t_0 \max_i |\frac{N_t(i)}{t} - \pi(i)| \le \eta $. That is to say, when the size of the stored samples is larger than $t_0$, the mixing time of the new data process with experience replay is $0$. Thus, the experience replay trick improves the mixing property of the data process, and hence improves the convergence rate.

	\textbf{Other factors that influence the finite sample bound:}
	(1) With the increasing of the feature norm $ L $, the finite sample bound increase. This is consistent with the empirical finding by \cite{dann2014policy} that the normalization of features is crucial for the estimation quality of GTD algorithms. (2) With the increasing of the feature dimension $ d $, the bound increase.  Intuitively, we need more samples for a linear approximation in a higher dimension feature space. 

	\section{Experiments}
	In this section, we report our simulation results to validate our theoretical findings. We consider the general convex-concave saddle problem, 
	\small\begin{equation} \label{simulation}
	\min_x\max_y\left( L(x,y) = \langle b-Ax,y \rangle + \frac{1}{2}\Vert x\Vert^2 - \frac{1}{2}\Vert y\Vert^2 \right)
	\end{equation}\normalsize
	where $ A $ is a $ n\times n $ matrix, b is a $ n \times 1 $ vector, Here we set $ n=10 $.	We conduct three experiment and set the step size to $ \alpha_t = c = 0.001 $, $\alpha_t = \mathcal{O} (\frac{1}{\sqrt{t}}) = \frac{0.015}{\sqrt{t}} $and $\alpha_t =  \mathcal{O}(\frac{1}{t} )= \frac{0.03}{t} $ respectively. In each experiment	we sample the data $ \hat{A}, \hat{b} $ three ways: sample from  two  Markov chains with different mixing time but share the same stationary distribution or sample  from  stationary distribution  i.i.d. directly.  We sample $ \hat{A} \text{ and } \hat{b}$ from Markov chain by using MCMC Metropolis-Hastings algorithms. Specifically, notice that the mixing time of a Markov chain is positive correlation with the second largest eigenvalue of its transition probability matrix  (\cite{levin2009markov}), we firstly conduct two transition probability matrix with different second largest eigenvalues( both with 1001 state and the second largest eigenvalue are 0.634  and 0.31   respectively), then using Metropolis-Hastings algorithms construct two Markov chain with same stationary distribution.	
	
	We run the gradient algorithm for the objective in (\ref{simulation}) based on the simulation data, without and with experience replay trick. The primal-dual gap error curves are plotted in Figure 1.
	
	We have the following observations. (1) The error curves converge in Markov setting with all the three setups of the step size. (2) The error curves with the data generated from the process which has small mixing time converge faster. The error curve for i.i.d. generated data converge fastest. (3) The error curve for different step size convergence at different rate. (4) With experience replay trick, the error curves in the Markov settings converge faster than previously.  All these observations are consistent with our theoretical findings.

	\begin{figure}[ht]
		\centering
		\subfigure[$ \alpha_t  = c$]{
			\label{figa}
			\includegraphics[width=1.5in,height=1.1in]{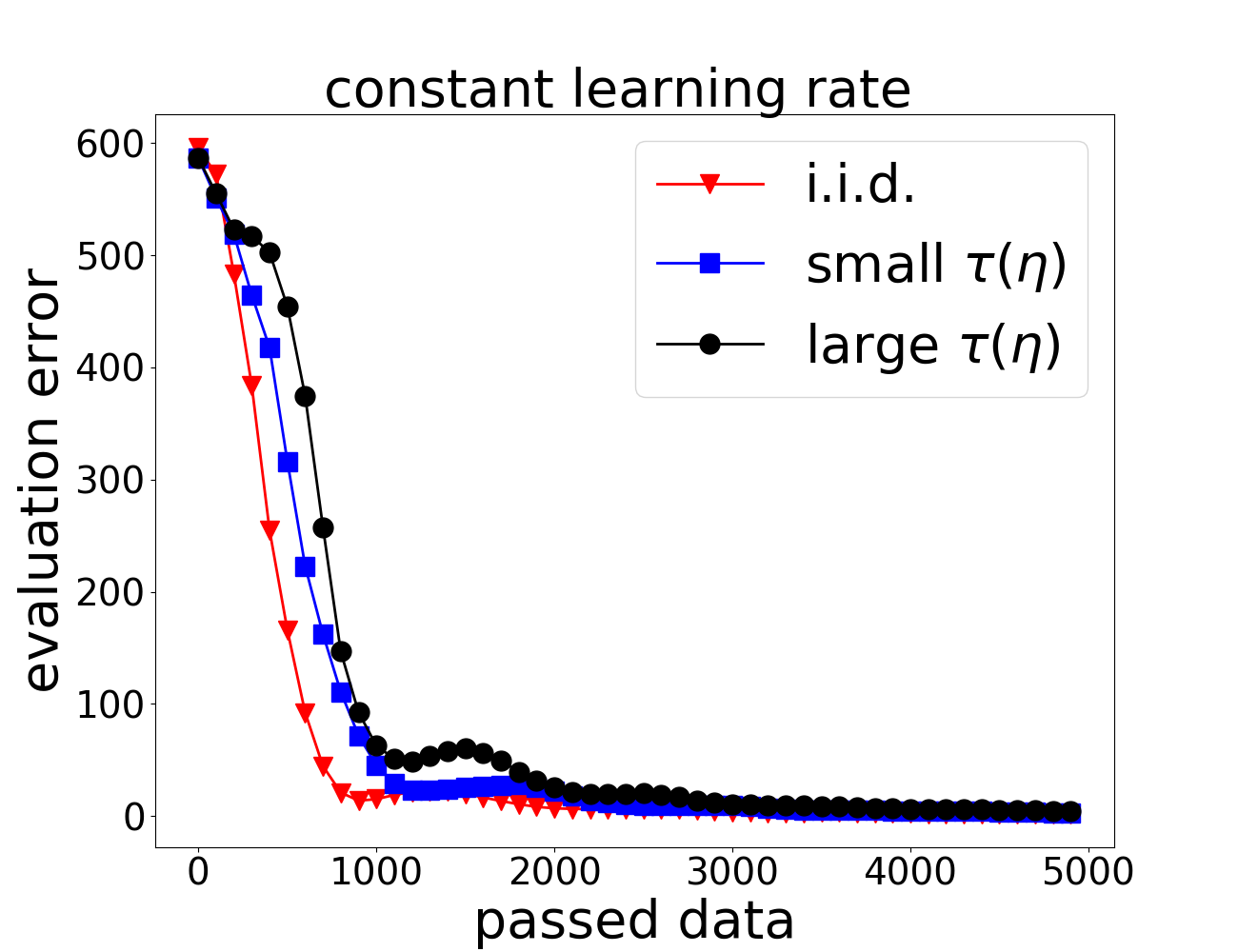}}
		\subfigure[$ \alpha_t =\mathcal{O} (\frac{1}{\sqrt{t}}) $]{
			\label{figc}
			\includegraphics[width=1.5in,height=1.1in]{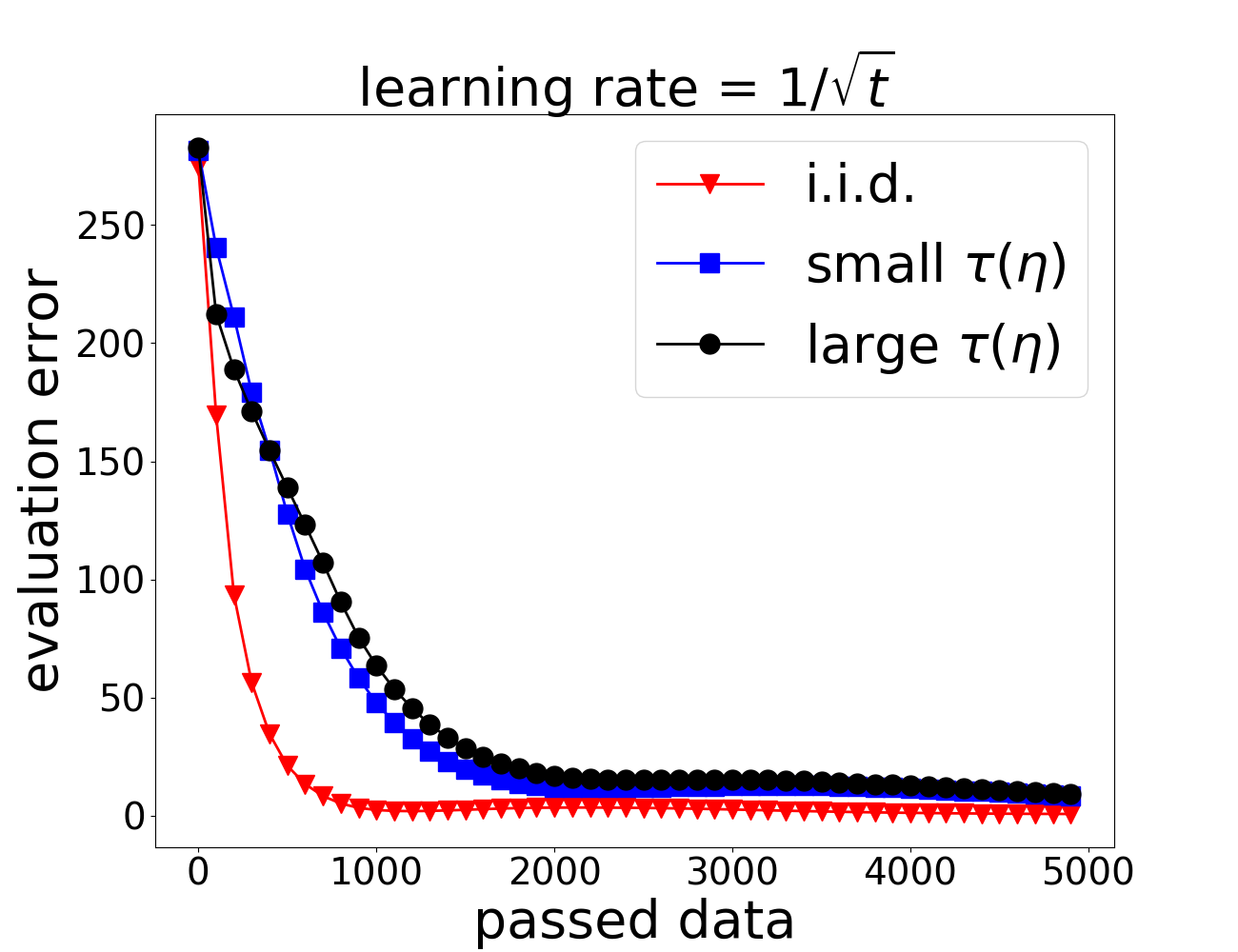}} 
		\subfigure[$ \alpha_t = \mathcal{O} {(\frac{1}{t})}$ ]{
			\label{fige}
			\includegraphics[width=1.5in,height=1.1in]{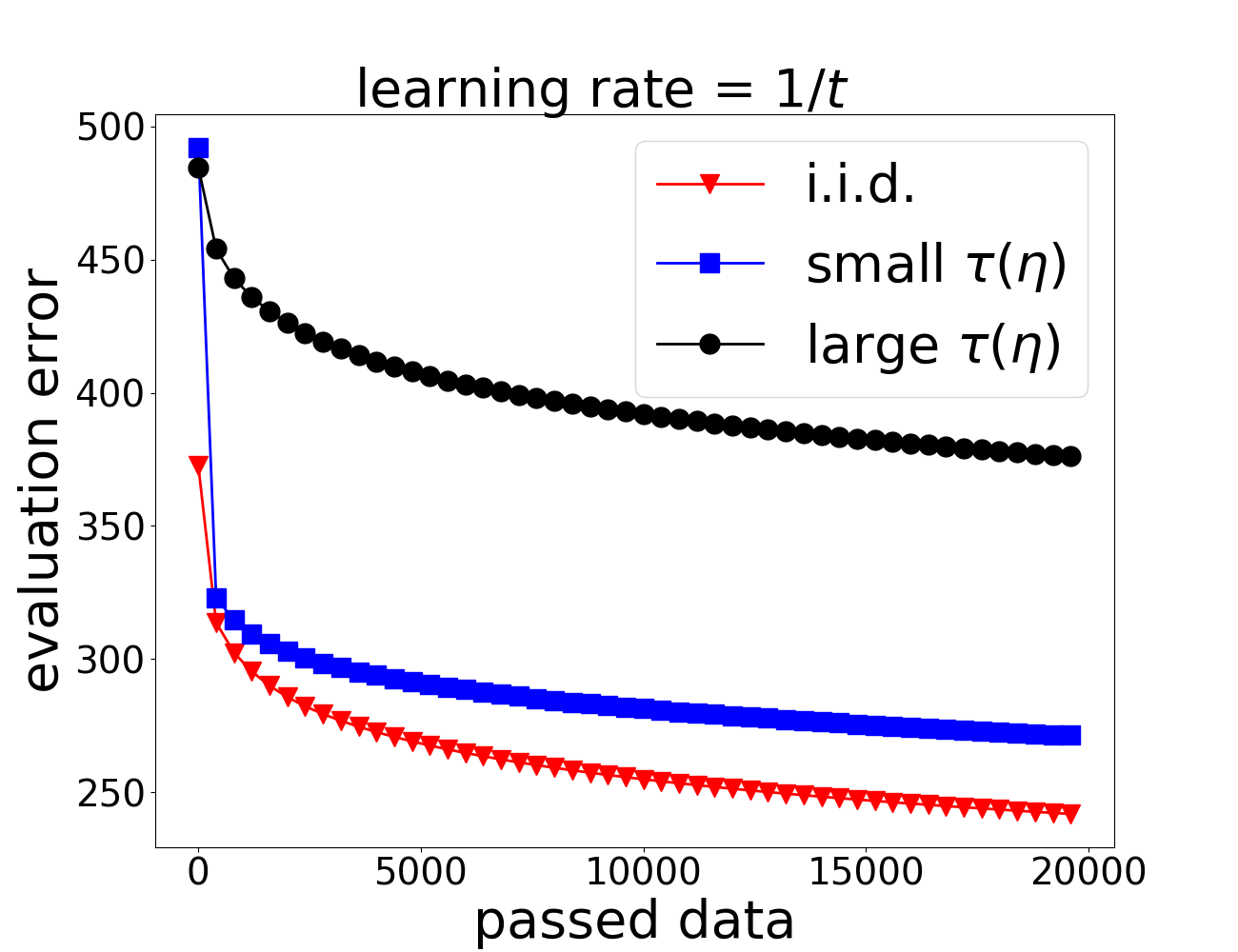}}	
		\subfigure[$ \alpha_t  = c$ with trick  ]{
			\label{figb}
			\includegraphics[width=1.5in,height=1.1in]{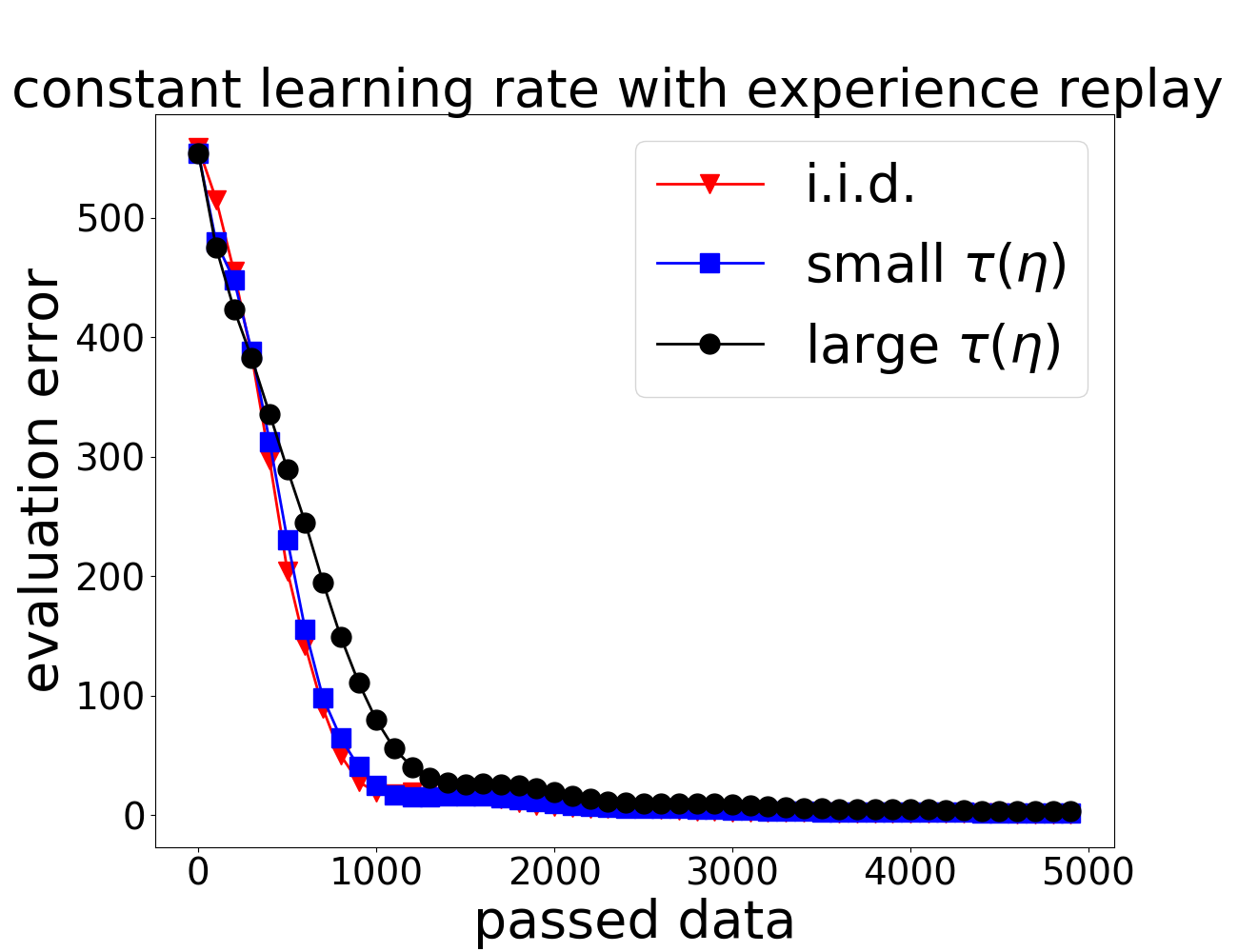}}
		\label{Fig1}	
		\subfigure[$ \alpha_t =\mathcal{O}( \frac{1}{\sqrt{t}})$ with trick ]{
			\label{figd}
			\includegraphics[width=1.5in,height=1.1in]{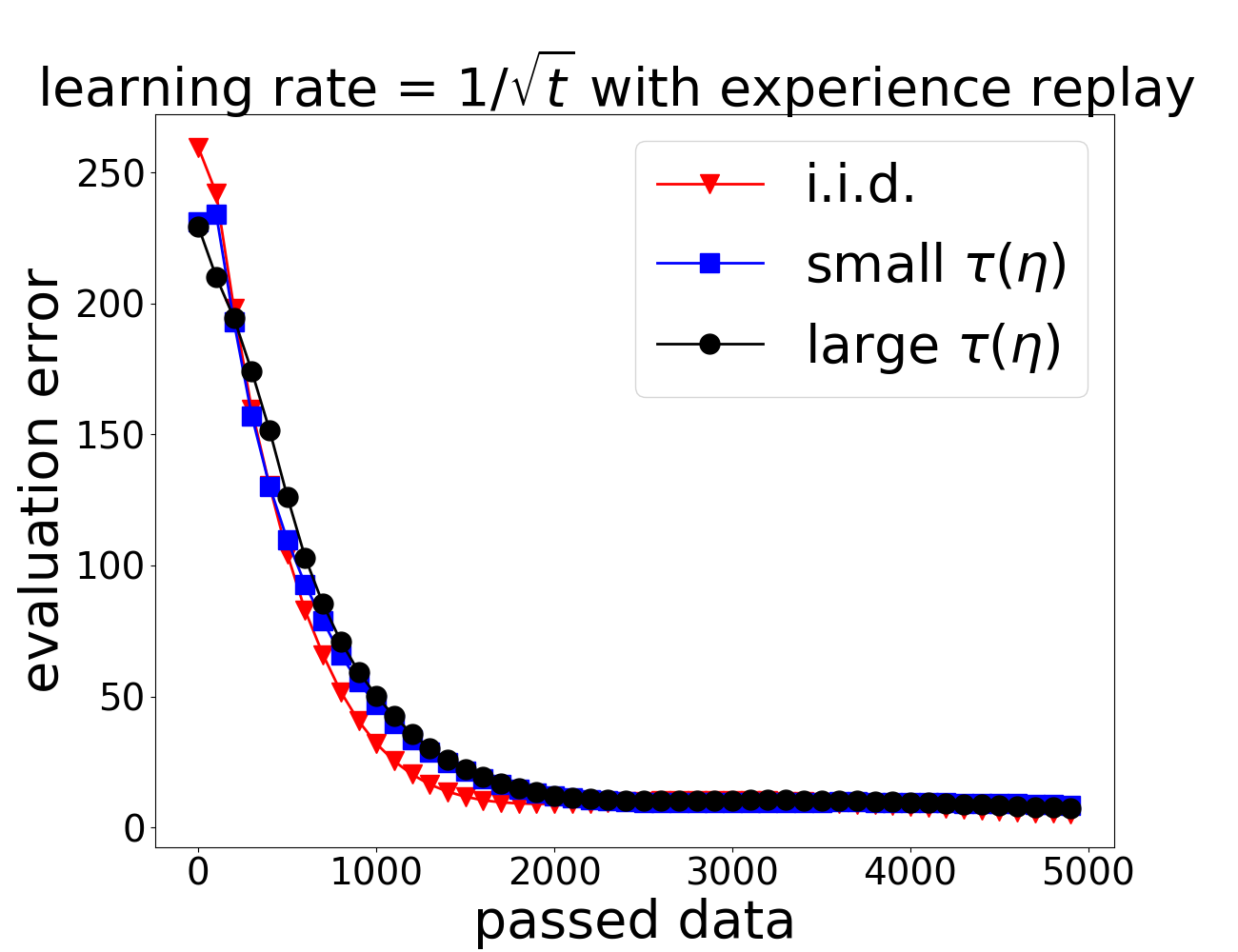}}	
		\subfigure[$ \alpha_t =\mathcal{O} {(\frac{1}{t})}$ with trick  ]{
			\label{figf}
			\includegraphics[width=1.5in,height=1.1in]{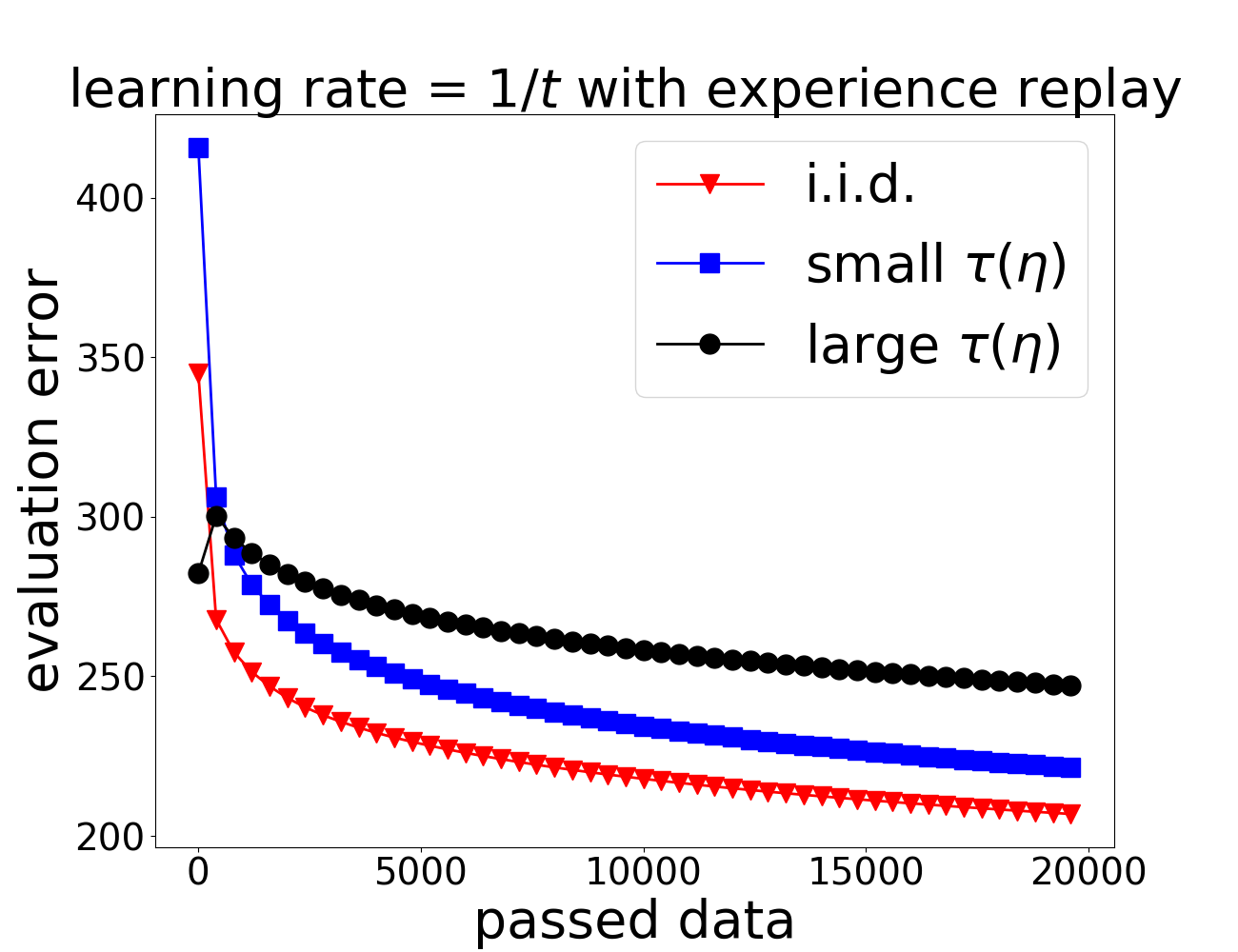}}		
		\caption{Experimental Results  }		
	\end{figure}	
	\section{Conclusion}
	In this paper, in the more realistic Markov setting, we proved the finite sample bound for the convex-concave saddle problems   with high probability and  in expectation. Then, we obtain the finite sample bound for GTD algorithms both in on-policy and off-policy, considering that the GTD algorithms are specific convex-concave saddle point problems. Our finite sample bounds provide important theoretical guarantee to the GTD algorithms, and also insights to improve them, including how to setup the step size and we need to improve the mixing property of the data like experience replay. In the future, we will study the finite sample bounds for policy evaluation with nonlinear function approximation.
	\section*{Acknowledgment}
	This work was supported by A Foundation for the Author of National Excellent Doctoral Dissertation of RP China (FANEDD 201312) and National Center for Mathematics and Interdisciplinary Sciences of CAS.

	\newpage
	\section{Supplementary}

This supplementary material  gives  the detail proof of main theorems in the main paper. The supplementary material is organized as follows:

Section \ref{s sec assumption} states some assumptions that corresponding to the main paper.
Section \ref{s sec proof} contains some lemmas and  the detail proofs of the  main Theorems in the main paper.
Section \ref{s sec proof of lemmas} contains detail proofs of lemmas used in proving the theorems.


\subsection{Assumptions}\label{s sec assumption}
\begin{assumption}[Bounded parameter space]\label{s assumptioncompactness}
	We assume there are finite  $ D< \infty $ such that	
	$$ \Vert z-z'\Vert \le D ~~for~ z,z'  \in \mathcal{X}  .$$
\end{assumption}

\begin{assumption}[Step size]\label{s assumptionstepsize}
	Let $ \lbrace\alpha_t \rbrace $ denote step size sequence which is non-increasing and for  $  \forall t ~~~~~ 0<\alpha_t \le \alpha_0 \le \infty $.
\end{assumption}

\begin{assumption}[Problem solvable]\label{s assrl nonsingular}
	The matrix $ A  $ and $ C $ are non-singular.
\end{assumption}

\begin{assumption}[Bounded data]\label{s assrl bound}
	The max norm of features are bounded  by $ L $, rewards are bounded by $ R_{max} $  and importance weights are bounded by $ \rho_{max} $.
\end{assumption}

\begin{assumption}[Lipschitz]\label{s assumptionlipschitz}
	For $ \Pi $-almost every $ \xi $ , the function $ \Phi(x,y,\xi) $ is Lipschitz for both x and y, that is there exists three constant $0 < L_{1x} <\infty $,$ 0 <  L_{1y}  <\infty $, $ 0 <  L_{1}  <\infty $ such that:
	$$ |\Phi(x',y,\xi) - \Phi(x,y,\xi)| \le L_{1x } \Vert x-x' \Vert~~ for ~~\forall x,x'\in \mathcal{X}_x$$
	$$ |\Phi(x,y,\xi) - \Phi(x,y',\xi)| \le L_{1y } \Vert y-y' \Vert~~ for ~~\forall y,y'\in \mathcal{X}_y$$
	
	and let $ L_{1  } \triangleq \sqrt{2}  \sqrt{L_{1x } ^2 + L_{1y }^2}$, we have 
	$$ |\Phi(z,\xi) - \Phi(z',\xi)| \le  L_1 \Vert z-z' \Vert~~ for ~~\forall z,z'\in \mathcal{X}_x \times \mathcal{X}_y.$$
	
\end{assumption}

\begin{assumption}[Smooth]\label{s assumptionsmooth}
	For $ \Pi $-almost every $ \xi $ , the partial gradient function  of $ \Phi(x,y,\xi) $ is Lipschitz,  that is there exists three constant $0 < L_{2x} <\infty $,$ 0 <  L_{2y}  <\infty $, $ 0 <  L_{2}  <\infty $ such that:
	$$ \Vert G_x(x',y,\xi) - G_x(x,y,\xi)\Vert \le L_{2x } \Vert x-x' \Vert~~ for ~~\forall x,x'\in \mathcal{X}_x,\forall y \in \mathcal{X}_y$$
	$$ \Vert G_y(x,y,\xi) - G_y(x,y',\xi)\Vert \le L_{2y } \Vert y-y' \Vert~~ for ~~\forall y,y'\in \mathcal{X}_y,\forall x \in \mathcal{X}_x$$
	Then, let $  L_{2 } \triangleq  \sqrt{2}\sqrt{L_{2x }^2+L_{2y }^2}$, we have
	$$ \Vert G (z,\xi) - G (z',\xi)\Vert \le  L_{2} \Vert z-z' \Vert~~ for ~~\forall z,z'\in \mathcal{X}_x \times \mathcal{X}_y.$$
\end{assumption}

\begin{assumption}\label{s assume mixing time}
	The mixing times of the stochastic process $ \lbrace\xi_t \rbrace$  are uniform in the sense that there exist  uniform mixing times $\tau(P,\eta)$ such that with probability 1 for all $ \eta >0 $ and $ s \in \mathbb{N} $
	\begin{equation*}
	\tau(P_{[s]},\eta)\le \tau(P, \eta).
	\end{equation*}
\end{assumption}


\subsection{Proofs}\label{s sec proof}
\subsubsection{Proof of Theorem 1}

\begin{theorem}\label{s highprobability}
	
	Suppose Assumption \ref{s assumptioncompactness}, \ref{s assumptionstepsize},\ref{s assumptionlipschitz},\ref{s assumptionsmooth} hold. For $\forall \delta>0$, and $\forall \eta > 0 $ such that $ \tau(\eta)\le T/2$, with probability at least $ 1-\delta$, we have
	\small
	\begin{align*}
	Err_\phi(\tilde{z}_1^T)  
	\le 		\frac{1}{\sum\limits_{t=1}^{T}\alpha_t} \left[ A  +B \sum_{t=1}^{T}\alpha_t^2 + C\tau(\eta)\sum_{t=1}^{T}\alpha_t^2  + F\eta\sum_{t=1}^{T}\alpha_t + H\tau(\eta) +  8DL_1 \sqrt{2\tau(\eta) \log\frac{\tau(\eta)}{\delta} \left(\sum_{t=1}^{T}\alpha_t^2 + \tau(\eta)\alpha_0 \right)} \right].
	\end{align*}
	where 
	\begin{small}
		\begin{align*}
		A = D^2   \qquad
		B = \frac{5}{2} L_1^2 \qquad
		C = 6L_1^2 + 2L_1L_2D \qquad
		F = 2L_1D  \qquad
		H = 6L_1D\alpha_0  .
		\end{align*} 
		
	\end{small}	
	
\end{theorem}

\begin{proof}[\textbf{Proof of Theorem \ref{s highprobability}}]
	For the convenience of proof, we introduce a sequence of  auxiliary variables $ \lbrace v _t \rbrace_{t=1,\dots,T-\tau}  $ that follow the iteration formula below respectively:
	\begin{align*}\numberthis\label{s iter v}
	&v _1 = z_1
	&v_{t+1} = P_{\mathcal{X}}\left(v_t - \alpha_t( g (z_t) - G(z_t,\xi_{t}))\right)\\
	\end{align*}
	Using the definition of error function and $ \tilde{z}_1^T $, we can convert expected error function to  a more friendly expression, from which we will start our analysis. Define $ \Gamma \triangleq \sum_{t=1}^{T}\alpha_t $.
	\begin{align*}		 
	Err_\phi(\tilde{z}_1^T)  &=	  	Err_\phi(\frac{1}{\Gamma}\sum_{t=1}^{T}\alpha_t z_t)    \\
	& =    \max_y\phi(\frac{1}{\Gamma}\sum_{t=1}^{T}\alpha_t x_t,y)-\min_x\phi(x,\frac{1}{\Gamma}\sum_{t=1}^{T}\alpha_t y_t)  \\
	& \overset{(1)}{=}    \max_y\phi(\frac{1}{\Gamma}\sum_{t=1}^{T}\alpha_t x_t,y) + \max_x-\phi(x,\frac{1}{\Gamma} \sum_{t=1}^{T} \alpha_t  y_t) \\
	& \overset{(2)}{\le}   \max_y\frac{1}{\Gamma}\sum_{t=1}^{T}\alpha_t \phi(x_t,y) + \max_x-\frac{1}{\Gamma}\sum_{t=1}^{T}\alpha_t \phi(x, 	 y_t)   \\
	&=   \max_x\max_y \frac{1}{\Gamma}\sum_{t=1}^{T}\alpha_t \left[ \phi( x_t,y)-\phi(x, y_t)\right] \\ 
	&	\overset{(3)}{\le} \max_x\max_y \frac{1}{\Gamma}\sum_{t=1}^{T}\alpha_t \left[(x_t-x)^\top g_x(x_t,y_t) +(y-y_t)^\top g_y(x_t,y_t) \right]   \\
	& =    \max_z \frac{1}{\Gamma}\sum_{t=1}^{T}\alpha_t \left[(z_t-z)^\top g(z_t) \right]  \numberthis\label{s err trans1} 
	\end{align*}
	(1) is a consequence of the Lemma \ref{s lemma-min}.\\
	(2)	follows by the convexity and concavity of $ \phi(x,y) $ with respect to  $ x $ and $ y $ respectively.\\
	(3) follows by the convexity-concavity once again.\\
	
	Notice that we can not bound the right hand side of (\ref{s err trans1}) directly because of the $ max $ operator and non-i.i.d. setting. 
	
	we rewrite   $ z = \argmax\limits_{z^*}  \frac{1}{\Gamma}\sum_{t=1}^{T}\alpha_t \left[(z_t-z^*)^\top g(z_t) \right]$,  it can be shown that $ z $ is measurable with respect to $ \huaf_T $. More specifically,   given $ \huaf_t $, $ t\le T, z$  is a random variable that correlated to $ z_t,\dots,z_T $. 
	
	Previous work that considers the saddle point problem under i.i.d. setting is easy to obtain the bound because they can utilize the i.i.d. property. In their setting, every sample based stochastic gradient function $ G(z_t,\xi_t) $ is unbiased  with respected to the $ g(z_t) $. Notice $ \E [ g(z_t) - G(z_t,\xi_t)  | \huaf_{t-1}] = 0 $. However, in our Markov setting this term cannot be arbitrary small.
	
	On the other hand, previous work that considers the Markov setting for convex function minimization problem is also easy to obtain the result since they are not bothered by random variable $ z $. In their problem   $ z $ is a constant rather than a  random variable which plays a key role when they try to handle the dependent and biased between sampling distribution.
	
	So we cannot apply existing techniques directly or even combine them trivially. 
	
	
	
	Now, to bound the right-hand side of (\ref{s err trans1}), we consider the following decomposition, for any $ \tau \ge 0$. In order to include special $ \tau = 0 $ case, we require : if $t1< t2$,   $ \sum_{t_2}^{t_1}(\cdot)  = 0$ 
	\begin{align*}
	&\sum_{t=1}^{T}\alpha_t \left[(z_t-z)^\top g(z_t) \right] \\
	=	&\sum\limits_{t=1}^{T-\tau}\alpha_t \left[(z_t-z)^\top g(z_t)  \right]  +   \sum_{t=T-\tau + 1}^{T}\alpha_t \left[(z_t-z)^\top g(z_t)  \right]\\
	=	&\sum\limits_{t=1}^{T-\tau}\alpha_t \left[(z_t-z)^\top (  G(z_t,\xi_{t}) -  G_x(z_t,\xi_{t}) 	+  g(z_t)    \right] \\
	&  +   \sum_{t=T-\tau + 1}^{T}\alpha_t \left[(z_t-z)^\top g(z_t)  \right],\\
	\end{align*}
	
	Denote $ g(z_t) - G(z_t,\xi_{t}) $ as $ \delta_t $ ,denote $ g (z_t) - G( z_t,\xi_{t+\tau} ) $ as $ \delta'_t $, denote $  G (z_t,\xi_{t+\tau}) -  G (z_t,\xi_{t})  $  as $ \delta''_t $, and recall the definition of $ v_t   $ ,
	
	\small\begin{align*}
	&   \max_z \sum_{t=1}^{T}\alpha_t \left[(z_t-z)^\top g(z_t) \right]      \numberthis\label{s keydecomposition} \\
	=	&   \max_z \left[ \sum\limits_{t=1}^{T-\tau}\alpha_t \left[\underbrace{(z_t-z)^\top  G (z_t,\xi_{t})}_{(a)} +\underbrace{(z_t-v_t)^\top \delta'_t}_{(b)} + \underbrace{(z_t-v_t)^\top \delta''_t}_{(c)} + \underbrace{(v_t-z)^\top \delta_t }_{(d)}    \right]
	+  \underbrace{ \sum_{t=T-\tau + 1}^{T}\alpha_t \left[(z_t-z)^\top g(z_t)\right]}_{(e)}\right]  .
	\end{align*} 	\normalsize

	In order to prove Theorem \ref{s expectationbound}, we need the following 7 lemmas, we prove these lemmas in section \ref{s sec proof of lemmas}.

	\begin{lemma}\label{s lemma-min}			
		\begin{equation}\label{s -min}
		-\min_x (f(x) )= \max_x (- f(x)).
		\end{equation}
	\end{lemma}

	\begin{lemma}\label{s keykeylemma}
		Let $ {\lbrace x_t\rbrace}  $ be a sequence of elements in  $ \mathbb{R}^n $, let $ \lbrace \Delta_t \rbrace $  be a sequence of element in $ \mathbb{R}^n $. Given the iteration formula 
		\begin{equation*}
		x_{t+1} = x_{t} +\Delta_t	.	
		\end{equation*} 
		Then for $ \forall x^* \in \mathbb{R}^n $ we have 
		\begin{equation}\label{s keykeyeq}
		2\langle x^*-x_t , \Delta_t \rangle  = \Vert \Delta_t \Vert^2 + \Vert  x_{t} -x^* \Vert^2 -\Vert  x_{t+1} - x^* \Vert^2.
		\end{equation}
		
		If the parameter space  $ \mathcal{X}$  is convex and $ \mathcal{X}\subset \mathbb{R}^n $ and using the projection step to constrain $ x  $ in the parameter space during the optimization path, that is 
		$$ x_{t+1} = P_{\mathcal{X}}( x_{t} +\Delta_t)	 $$			
		\begin{equation}\label{s keykeyeq2}
		Then \qquad
		2\langle x^*-x_t , \Delta_t \rangle  \le \Vert \Delta_t \Vert^2 + \Vert  x_{t} -x^* \Vert^2 -\Vert  x_{t+1} - x^* \Vert^2.
		\end{equation}
		
	\end{lemma}

	\begin{lemma}\label{s lemma alphaz}
		If $ z_t $ follows the iteration formula : $ 	z_{t+1} =	\mathcal{P}_{\mathcal{X} }( z_t - \alpha_t(G(z_t,\xi_t))) $, then:
		$$  \Vert  \alpha_{t_1}z_{t_1} -\alpha_{t_2 +1}z_{t_2+1}  \Vert  \le	D  (\alpha_{t_1 } - \alpha_{t_2+1}  )	
		+  L_{1}(t_2 - t_1  )\alpha_{t_1 }^2  .$$
	\end{lemma}

	\begin{lemma}\label{s terma}
		If Assumption \ref{s assumptionstepsize}, \ref{s assumptioncompactness} and \ref{s assumptionlipschitz}   hold,
		\begin{align*}
		\max_z \sum\limits_{t=1}^{T-\tau}\alpha_t \left[(z_t-z)^\top  G(z_t,\xi_{t})\right]    \le \frac{1}{2}D^2 + \frac{1}{2}L_{1}^2\sum_{t=1}^{T-\tau}\alpha_t^2 .
		\end{align*}	
		
	\end{lemma}

	\begin{lemma}\label{s termd}
		If Assumption \ref{s assumptionstepsize}, \ref{s assumptioncompactness} and \ref{s assumptionlipschitz}   hold,
		\begin{align*}
		\max_z \sum\limits_{t=1}^{T-\tau}\alpha_t (v_t-z)^\top \delta_t    \le \frac{1}{2}D^2 + 2L_{1}^2\sum_{t=1}^{T-\tau}\alpha_t^2.
		\end{align*}	
		
	\end{lemma}
	
	\begin{lemma}\label{s martingale}
		Let Assumption\ref{s assumptioncompactness}, \ref{s assumptionstepsize}, \ref{s assumptionlipschitz} and \ref{s assumptionsmooth} hold and $ \delta \in (0,1) $	With probability $ 1-\delta $ :
		
		\begin{align*}
		\sum_{t=1}^{T-\tau} \left[\alpha_t(z_t-v_t)^\top (g(z_t) - G(z_t,\xi_{t+\tau})) \right]\le  8DL_1  \sqrt{2\tau \log\frac{\tau}{\delta} \left(\sum_{t=1}^{T}\alpha_t^2 + \tau\alpha_0 \right)  } + 2DL_1\sum_{t=1}^{T}\alpha_t\int\vert \pi(\xi) -p_{[t]}^{t+\tau}(\xi) \vert d\xi  .
		\end{align*}
		
	\end{lemma}

	\begin{lemma}\label{s termc}
		If Assumption \ref{s assumptionstepsize},\ref{s assumptioncompactness}, \ref{s assumptionlipschitz} and \ref{s assumptionsmooth} hold,
		\begin{align*}
		& \sum_{t=1}^{T-\tau}\left[\alpha_t(z_t-v_t)^\top (G(z_t,\xi_{t+\tau}) - G(z_t,\xi_{t})) \right]  
		\le	2L_{1}\tau(DL_{2} +3L_{1})\sum_{t=\tau+1}^{T-\tau}\alpha_{t-\tau}^2 +  5\tau DL_{1}\alpha_0 .
		\end{align*}
		
	\end{lemma}

	Now we are ready to give the proof of Theorem \ref{s expectationbound} by constructing a novel decomposition and then use   above lemmas.

	Firstly, we apply Lemma \ref{s terma} to bound the term (a), \\
	Then, we apply Lemma \ref{s martingale}  to bound the term(b),\\
	Then, we apply Lemma \ref{s termc} to bound the term(c),\\
	Then, we apply Lemma \ref{s termd} to bound the term(d),\\
	
	Finally,for term (e), notice that
	\begin{align*}
	\E\left[\max_z\sum_{t=T-\tau + 1}^{T}\alpha_t \left[(z_t-z)^\top g(z_t)\right]\right]	\le \tau D L_{1} \alpha_0	,
	\end{align*}

	Combine   above five terms, and for a given $ \eta $ , we set $ \tau  = \tau(\eta) $ by Assumption \ref{s assume mixing time}, then the following bound hold with probability $ 1-\delta $:
	\begin{align*}
	& Err_\phi(\tilde{z}_1^T) \\
	\le&  \frac{1}{\sum_{t=1}^{T}\alpha_t}\Bigg[ D^2 + \frac{5}{2}L_{1}^2 \sum_{t=1}^{T }\alpha_t^2 + 2DL_{1}\eta\sum_{t=1}^{T }\alpha_t + 2L_{1} \tau(\eta) (3L_{1} + DL_{2})\sum_{t= 1}^{T }\alpha_{t}^2 + 6\tau(\eta)DL_{1}\alpha_0 \\ &  +   8DL_1  \sqrt{2\tau(\eta) \log\frac{\tau(\eta)}{\delta} \left(\sum_{t=1}^{T}\alpha_t^2 + \tau(\eta)\alpha_0 \right)  }    \Bigg]\\
	\le&	\frac{1}{\sum\limits_{t=1}^{T}\alpha_t} \Bigg[ A  +B \sum_{t=1}^{T}\alpha_t^2 + C\tau(\eta)\sum_{t=1}^{T}\alpha_t^2  + F\eta\sum_{t=1}^{T}\alpha_t + H\tau(\eta)
	+  8DL_1 \sqrt{2\tau(\eta) \log\frac{\tau(\eta)}{\delta} \left(\sum_{t=1}^{T}\alpha_t^2 + \tau(\eta)\alpha_0 \right)} \Bigg]
	\end{align*}
	where 
	\begin{align*}
	A = D^2   \qquad
	B = \frac{5}{2} L_1^2 \qquad
	C = 6L_1^2 + 2L_1L_2D \qquad
	F = 2L_1D  \qquad
	H = 6L_1D\alpha_0  .
	\end{align*} 
	
\end{proof}

\subsubsection{Proof of Corollary 1}

\begin{corollary}\label{s expectationbound}
	
	Suppose Assumption \ref{s assumptioncompactness},\ref{s assumptionstepsize},\ref{s assumptionlipschitz},\ref{s assumptionsmooth} hold, then for the gradient algorithm optimizing the convex-concave saddle point problem,  $\forall \eta > 0 $ such that $ \tau(\eta)\le T/2$, we have 
	\small\begin{align*}
	\E_{\mathcal D} [Err_\phi(\tilde{z}_1^T)] 
	\le 		\frac{1}{\sum_{t=1}^{T}\alpha_t} \left[ A  +B \sum_{t=1}^{T}\alpha_t^2 + C\tau(\eta)\sum_{t=1}^{T}\alpha_t^2  + F\eta\sum_{t=1}^{T}\alpha_t + H\tau(\eta) \right], \forall \eta>0,
	\end{align*}\normalsize

\end{corollary}

Firstly, we give a key lemma that will be used in the proof of Corollary \ref{s highprobability}.

\begin{lemma}\label{s termb}
	Let Assumption \ref{s assumptioncompactness},\ref{s assumptionstepsize}, \ref{s assumptionlipschitz} and \ref{s assume mixing time} hold, $ \tau \ge 0  $,  recall the definition of $ v_t $  :
	\begin{align*}
	\E\left[ \sum_{t=1}^{T-\tau}\alpha_t(z_t-v_t)^\top \delta'_t \right] &\le  2DL_{1}\sum_{t=1}^{T-\tau}\alpha_t\E\int\vert \pi(\xi) -p_{[t]}^{t+\tau}(\xi) \vert d\xi .
	\end{align*}		
\end{lemma}

\begin{proof}[\textbf{Proof of Theorem \ref{s highprobability}}]
	The result can be obtained by replacing the term (b) in decomposition (\ref{s keydecomposition}) by an expectation upper bound using Lemma \ref{s termb} and using Lemma \ref{s terma}, \ref{s termc}, \ref{s termd} once again to bound the rest term.
\end{proof}

\subsubsection{Proof of theorem 2}

\begin{theorem}\label{s gtdbound} 
	Suppose assumptions 1-4 hold, then we have the following finite sample bounds for the  error $ \Vert V- \tilde{v}_1^T \Vert_\pi $ in GTD algorithms. In on-policy case, the expectation bound is \small$\mathcal{O} \left(\frac{L \sqrt{L^4 d^3 \lambda_M\pi_{max} (1+\tau(\eta))\pi_{max}o_1}}{\nu_C}\right)$\normalsize and the high probability bound is \small$\mathcal{O}\left(\frac{\sqrt{  L^4 d^2 \lambda_M\pi_{max} }}{\nu_C} \left( \sqrt{(1+\tau(\eta))L^2 d o_1 + \log\left(\frac{\tau(\eta)}{\delta} \right)o_2} \right) \right)$\normalsize;		
	In off-policy case, the expectation bound is \small$\mathcal{O} \left(\frac{L^2d\sqrt{2\lambda_C\lambda_M\pi_{max} (1+\tau(\eta))o_1}}{\nu_{(A^TM^{-1}A)}}\right)$\normalsize anf the high probability bound is \small$\mathcal{O} \left(\frac{\sqrt{2\lambda_C\lambda_M\pi_{max} }}{\nu_{(A^TM^{-1}A)}}  \left( \sqrt{L^4d^2(1+\tau(\eta))o_1 + \log{(\frac{\tau(\eta)}{\delta})}o_2 } \right) \right)$\normalsize, 
	where $ \nu_C$ is the smallest eigenvalue of the $C$ , $ o_1 =   (\frac{\sum_{t=1}^{T}\alpha_t^2}{\sum_{t=1}^{T}\alpha_t}) , o_2 =   (\frac{\sqrt{\sum_{t=1}^{T}\alpha_t^2}}{\sum_{t=1}^{T}\alpha_t}) $. 		
	
\end{theorem}

\begin{lemma}\label{s prop gm bound}
	Assumption \ref{s assrl bound},  implies Assumption   \ref{s assumptionlipschitz},\ref{s assumptionsmooth}, and
	\begin{align*}
	L_{1}^2 & \le   2(   2\Vert A \Vert D  +   \Vert b \Vert   +   \lambda_M  D   )^2 ,\\
	L_{2}^2 & \le   2(2\Vert A \Vert +\lambda_M)^2 .				
	\end{align*}
	
\end{lemma}

Here we state three results from \cite{liu2015finite}.
\begin{lemma}[Lemma 2 of \cite{liu2015finite}]\label{s lemma norm bound}
	For $ \forall \xi $,	the l2-norm of matrix $ \hat{A}_t $ and the l2-norm of vector $ \hat{b}_t $ are bounded by 
	\begin{align*}
	&\Vert \hat{A}_t \Vert_2 \le (1+\gamma)\rho_{max}L^2d, & \Vert \hat{b}_t \Vert_2 \le\rho_{max}L R_{max}\numberthis.
	\end{align*}
\end{lemma}

\begin{lemma}[Proposition 4 of \cite{liu2015finite}]
	Let V be the value of the target policy and $ \tilde{v}_1^T = \Phi\tilde{x}_1^T $, where $ \tilde{x}_1^T$ is the value function returned by on policy GTD algorithms. Then  we have 
	$ \Vert V-\tilde{v}_1^T \Vert_\pi \le \frac{1}{1-\gamma}\left( \Vert V - \Pi V \Vert_\pi  + \frac{L}{\nu_C}\sqrt{2d\lambda_M\pi_{max}Err(\tilde{x}_1^T,\tilde{y}_1^T)} \right) $.
\end{lemma}
\begin{lemma}[Proposition 5 of \cite{liu2015finite}]
	Let V be the value of the target policy and $ \tilde{v}_1^T = \Phi\tilde{x}_1^T $, where $ \tilde{x}_1^T$ is the value function returned by off policy GTD algorithms. Then   we have 
	$ \Vert V-\tilde{v}_1^T \Vert_\pi \le \frac{1+\gamma\sqrt{\rho_{max}}}{1-\gamma} \Vert V - \Pi V \Vert_\pi  + \sqrt{\frac{2\lambda_C \lambda_M \pi_{max}}{\nu_{A^\top M^{-1}A}}Err(\tilde{x}_1^T,\tilde{y}_1^T)}  $.
\end{lemma}

\begin{proof}[\textbf{Proof of Theorem \ref{s gtdbound}}]
	Substitute  Lemma \ref{s lemma norm bound} into Lemma \ref{s prop gm bound} yields Proposition 1 in the main paper. Then using Proposition 1 together with  Proposition 4-5 of \cite{liu2015finite} we can obtain the Theorem \ref{s gtdbound}.
	
\end{proof}

\subsection{Proof of lemmas}\label{s sec proof of lemmas}

\begin{proof}[\textbf{Proof of lemma \ref{s lemma-min}}]
	Firstly, we will proof   $  -\min_x (f(x)) \le \max_x (- f(x)) $    :
	\begin{equation}    
	\begin{aligned}
	&\forall x \ \ &\max(-f(x)) &\ge -f(x)  \\
	&&f(x) &\ge -\max(-f(x))                \\
	&\text{so}  &\min(f(x)) &\ge -\max(-f(x))       \\
	&&  -\min(f(x)) &\le \max(-f(x))
	\end{aligned}
	\end{equation}  
	
	Then, we will proof  $  -\min_x (f(x)) \ge \max_x (- f(x)) $    :
	
	\begin{equation}    
	\begin{aligned}
	&\forall x \ \ &\min( f(x) ) &\ge f(x)  \\
	&& -\min( f(x) ) &\le -f(x)     \\
	&\text{so} &-\min_x f(x) &\ge \max_x (-f(x))    \\
	\end{aligned}
	\end{equation}  
	Combine the above inequality, we get the result (\ref{s -min}).
\end{proof}

\begin{proof}[\textbf{Proof of lemma \ref{s keykeylemma}}]
	Using the iteration formula together with the definition of L2-norm, we can get for  $ \forall x^* \in \mathbb{R}^n $
	\begin{align*}
	\Vert x_{t+1} - x^* \Vert^2 -  \Vert x_{t}-x^* \Vert^2 &= \Vert x_{t} + \Delta_t - x^* \Vert^2 -  \Vert x_{t}-x^* \Vert^2\\
	&=  \Vert x_{t} - x^* \Vert^2 + 2\langle  x_t-x^*, \Delta_t \rangle  + \Vert \Delta_t \Vert^2 -  \Vert x_{t}-x^* \Vert^2
	\end{align*}            
	So equation \ref{s keykeyeq} can be obtained by moving the second term from right hand to left hand. And inequation \ref{s keykeyeq2} hold because projection is a contraction mapping with respect to $ \mathcal{ L}_2 $ norm.
	
\end{proof}

\begin{proof}[\textbf{Proof of lemma \ref{s lemma alphaz}}]
	\begin{align*}
	&   \Vert  \alpha_{t_1}z_{t_1} -\alpha_{t_2 +1}z_{t_2+1}  \Vert  \\
	=   &   \Vert \sum_{s=t_1}^{t_2} \alpha_sz_s -\alpha_{s+1}z_{s+1}   \Vert  \\
	\le &   \Vert  \sum_{s=t_1}^{t_2} \alpha_sz_s -\alpha_{s+1}(z_{s} - \alpha_sG(z_s,\xi_s))   \Vert  \\
	\le &   \Vert \sum_{s=t_1}^{t_2} \alpha_s z_s -\alpha_{s+1}z_{s}    \Vert  +    \Vert \sum_{s=t_1}^{t_2} ( \alpha_{s+1} \alpha_sG(z_s,\xi_s))   \Vert    \\
	\le &  \sum_{s=t_1}^{t_2}   \Vert \alpha_s  -\alpha_{s+1}   \Vert   \Vert  z_{s}    \Vert +     \Vert \sum_{s=t_1}^{t_2} ( \alpha_{s+1} \alpha_sG(z_s,\xi_s))   \Vert   \\
	\le &   D \sum_{s=t_1 }^{t_2 }  ( \alpha_s   -   \alpha_{s+1}  )  +  \sum_{s=t_1}^{t_2} ( \alpha_{s+1} \alpha_s  \Vert G(z_s,\xi_s)\Vert)    \\
	\le & D  (\alpha_{t_1 } - \alpha_{t_2+1}  ) 
	+  L_{1}(t_2 - t_1  )\alpha_{t_1 }^2 .
	\end{align*}
\end{proof}

\begin{proof}[\textbf{Proof of lemma \ref{s terma}}]
	
	Applying Lemma \ref{s keykeylemma} by setting  $ x_t = z_t $ , $ x^* = z $ and $ \Delta_t = -\alpha_t G(z_t,\xi_t) $  can get  :
	\begin{align*}
	&\sum\limits_{t=1}^{T-\tau}\alpha_t \left[(z_t-z)^\top  G(z_t,\xi_{t})\right] \\ 
	\le     &\sum\limits_{t=1}^{T-\tau} \frac{1}{2}\left[\Vert \alpha_t G(z_t,\xi_t) \Vert^2  + \Vert z_t - z \Vert ^2 - \Vert z_{t+1}  - z\Vert^2 \right]\\
	\le     &\frac{1}{2} \Vert z_1 - z \Vert^2 + \frac{1}{2}L_{1}^2  \sum\limits_{t=1}^{T-\tau}  \alpha_t^2
	\end{align*}
	The   statement follows by taking $ max $ on the first term directly.
	
\end{proof}

\begin{proof}[\textbf{Proof of lemma \ref{s termd}}]
	The proof is entirely similar to the proof of \ref{s terma} if we  set $ x_t = v_t $ , $ x^* = z $ and $ \Delta_t = -\alpha_t \delta_t $.
\end{proof}\

\begin{proof}[\textbf{Proof of lemma \ref{s martingale}}]
	We construct a family of martingales, each of which we control with high probability. We begin by defining the following random variables
	\begin{align*}
	&   A_t \triangleq \alpha_{t-\tau}(z_{t-\tau }-v_{t-\tau })^\top (g(z_{t-\tau  }) - G(z_{t-\tau },\xi_{t}))\\               
	&   \sum_{t=\tau+1}^{T}A_t =\sum_{t=1}^{T-\tau}\E\left[\alpha_t(z_t-v_t)^\top (g(z_t) - G(z_t,\xi_{t+\tau})) \right]                    
	\end{align*} 
	Define the filtration of $ \sigma- $fields $ \mathcal{A}_i^j =\huaf_{\tau i + j}$ for $ j = 1, \dots , \tau $. Then we can construct $ \tau $ sets of martingales $ \lbrace B_1^j,B_2^j,\dots \rbrace $ for $ j = 1,2,\dots,\tau $ :
	\begin{align*}
	B_i^j = A_{i\tau + j} - \E\left[A_{i\tau + j} |\mathcal{A}_{i-1}^{j}\right]\\
	\end{align*}
	By definition, $ B_i^j $  is measurable with respect to $ \mathcal{A}_i^j $, and $ \E[B_i^j|\mathcal{A}_{i-1}^j] =0 $. So for each j, the sequence $ \lbrace B_i^j ,i=1,2,\dots \rbrace $ is a martingale difference sequence adapted to the filtration $ \mathcal{A}_i^j $.
	For a fixed  $ j_0 \in 1,2,\dots, \tau $, the index $ i $ for martingale sequence $ B_i^{j_0} $ can take value from the index set $ \mathcal{I}(j_0) $,\
	$$ \mathcal{I}(j) =  \left\lbrace \begin{array}{lr}
	\mathcal{I}_1 =\lbrace 1,\dots,\lfloor T/\tau \rfloor +1 \rbrace  \quad &if \quad  j\le T - \tau\lfloor T/\tau \rfloor\\ 
	\mathcal{I}_2 = \lbrace 1,\dots,\lfloor T/\tau \rfloor \rbrace  \quad &if \quad  j > T - \tau\lfloor T/\tau \rfloor\\ 
	\end{array} \right. $$
	
	\begin{align*}
	\sum_{t=\tau+1}^{T}A_t = \sum_{j=1}^{\tau}\sum_{i\in \mathcal{I}(j)} B_i^j + \sum_{t=\tau+1}^{T}\E[A_t|\huaf_{t-\tau}]\numberthis\label{s At}
	\end{align*}
	
	Notice
	\begin{align*}
	|B_i^j| = |A_{i\tau + j} - \E\left[A_{i\tau + j} |\mathcal{A}_{i-1}^{j}\right]| \le|A_{i\tau + j}| + |\E\left[A_{i\tau + j} |\mathcal{A}_{i-1}^{j}\right]|\le 8DL_1\alpha_{i\tau+j}
	\end{align*}
	So applying the triangle inequality  and Azuma's inequality, we can bound  the martingale difference sequence term of \ref{s At} :
	
	\begin{align*}
	P\left(\sum_{j=1}^{\tau}\sum_{i\in \mathcal{I}(j)} B_i^j > \gamma \right)\le \sum_{j=1}^{\tau}  P\left(\sum_{i\in \mathcal{I}(j)} B_i^j > \frac{\gamma}{\tau}\right)\le \sum_{j=1}^{\tau}\exp\left(-\frac{\gamma^2}{128D^2L_1^2\tau^2\sum_{i\in \mathcal{I}(j)} \alpha_{i\tau }^2 }\right)
	\end{align*}
	
	Notice  that $ \tau(\alpha_{i\tau} )^2\le \sum_{j=1}^{\tau}(\alpha_{(i-1)\tau+j} )^2   $. So   $   \sum_{i\in \mathcal{I}(j)} \tau(\alpha_{i\tau })^2  \le \sum_{t=1}^{T}\alpha_t^2 + \tau\alpha_0~ \text{ for } ~\forall j $
	
	Setting $ \gamma = 8DL_1 \sqrt{2\tau\log\frac{\tau}{\delta}  (\sum_{t=1}^{T}\alpha_t^2 + \tau\alpha_0) } $, we get
	$$                  P\left(\sum_{j=1}^{\tau}\sum_{i\in \mathcal{I}(j)} B_i^j > 8DL_1 \sqrt{2\tau \log\frac{\tau}{\delta} \left(\sum_{t=1}^{T}\alpha_t^2 + \tau\alpha_0 \right)} \right)\le \delta $$

	For the last term of \ref{s At}, recall the Lemma \ref{s termb}
	\begin{align*}
	|\E[A_t|\huaf_t-\tau]| \le 2DL_1   \int\vert \pi(\xi) -p_{[t]}^{t+\tau}(\xi) \vert d\xi 
	\end{align*}
	
	combine the above bound completes the proof.
\end{proof}

\begin{proof}[\textbf{Proof of lemma \ref{s termc}}]
	Rearrange the left hand side of the above inequation, we can get
	\begin{align*}
	&    \sum_{t=1}^{T-\tau}\left[\alpha_t(z_t-v_t)^\top (G(z_t,\xi_{t+\tau}) - G(z_t,\xi_{t})) \right] \\
	=   &   \underbrace{\sum_{t=1+\tau}^{T-\tau}\left[\alpha_{t-\tau}(z_{t-\tau}-v_{t-\tau})^\top G(z_{t-\tau} , \xi_{t})- \alpha_t(z_{t}-v_{t})^\top G(z_{t} , \xi_{t}) \right]}_{(f)} \\
	&  +  \sum_{t = T-\tau}^{T}\alpha_{t-\tau}(z_{t-\tau}-v_{t-\tau})^\top G(z_{t-\tau} , \xi_{t}) -  \sum_{t = 1}^{\tau}\alpha_t(z_{t}-v_{t})^\top G(z_{t} , \xi_{t}) \\
	\end{align*}
	Considering the term $ (f) $:
	
	\begin{align*}
	&    \sum_{t=1+\tau}^{T-\tau}\left[\alpha_{t-\tau}(z_{t-\tau}-v_{t-\tau})^\top G(z_{t-\tau} , \xi_{t})- \alpha_t(z_{t}-v_{t})^\top G(z_{t} , \xi_{t}))\right]\\
	=   &    \sum_{t=1+\tau}^{T-\tau}\alpha_{t-\tau}(z_{t-\tau}-v_{t-\tau})^\top (G(z_{t-\tau} ,  \xi_{t})- G(z_{t} , \xi_{t})) \\ 
	&   + \sum_{t=1+\tau}^{T-\tau}[\alpha_{t-\tau}(z_{t-\tau}-v_{t-\tau}) - \alpha_t(z_{t}-v_{t})] ^\top G(z_{t} , \xi_{t})\\
	\le &    \sum_{t=1+\tau}^{T-\tau}\alpha_{t-\tau} \Vert z_{t-\tau}-v_{t-\tau}  \Vert \Vert G(z_{t-\tau}  , \xi_{t})- G(z_{t}  , \xi_{t})\Vert \\ 
	&   +  \sum_{t=1+\tau}^{T-\tau}(\alpha_{t-\tau}z_{t-\tau}- \alpha_t z_{t} )^ \top  G(z_{t}  ,  \xi_{t}) +\sum_{t=1+\tau}^{T-\tau} (\alpha_t v_{t} -\alpha_{t-\tau} v_{t-\tau} )^\top G(z_{t}  , \xi_{t})    \\
	\overset{(1)}{\le} &L_{2} \sum_{t=1+\tau}^{T-\tau}\alpha_{t-\tau}\Vert z_{t-\tau}-v_{t-\tau}\Vert \Vert \sum_{s=t-\tau}^{t-1}(z_{s} - z_{s+1}) \Vert\\
	&   + L_{1} \sum_{t=1+\tau}^{T-\tau}  \Vert \alpha_{t-\tau}z_{t-\tau}- \alpha_{t} z_{t}\Vert  + L_{1}\sum_{t=1+\tau}^{T-\tau}\alpha_{t-\tau} \Vert   v_{t} - v_{t-\tau}\Vert  \\
	\overset{ }{\le} &  2DL_{1}L_{2}\tau \sum_{t=\tau+1}^{T-\tau}\alpha_{t-\tau}^2 + 2L_{1}^2\tau\sum_{t=1+\tau}^{T-\tau}\alpha_{t-\tau}^2 + 4L_{1}^2\tau\sum_{t=1+\tau}^{T-\tau}\alpha_{t-\tau}^2 +   \tau   \alpha_0 DL_1\\
	\overset{ }{\le} &  2L_{1}\tau(3L_{1}+DL_{2} )\sum_{t=\tau+1}^{T-\tau}\alpha_{t-\tau}^2 +   \tau   \alpha_0 DL_1
	\end{align*}
	The first term of (1) can be bounded by using the iteration formula of $ z_t $.\\ 
	The second term of (1) can be bounded by using the Lemma \ref{s lemma alphaz}.\\
	The third term of (1) can be bounded by using the iteration formula of $ v_t $.
	
	In conclusion, we  get the Lemma \ref{s termc}
	\begin{align*}
	&    \sum_{t=1}^{T-\tau}\left[\alpha_t(z_t-v_t) (G(z_t,\xi_{t+\tau}) - G(z_t,\xi_{t})) \right] \\
	\le &       2L_{1}\tau(3L_{1}+DL_{2})\sum_{t=\tau+1}^{T-\tau}\alpha_{t-\tau}^2 +  \tau \alpha_0 DL_1 + 2DL_{1}\sum_{t=1}^{\tau}\alpha_t + 2DL_{1}\sum_{t=T-\tau}^{T}\alpha_{t-\tau} \\
	= &             2L_{1}\tau(3L_{1}+DL_{2})\sum_{t=\tau+1}^{T-\tau}\alpha_{t-\tau}^2 + 5 \tau \alpha_0 D L_1 .
	\end{align*}
	
	
\end{proof}

\begin{proof}[\textbf{Proof of lemma \ref{s termb}}]
	\begin{align*}      
	&\E\left[  (z_t-v_t)^\top \delta'_t | \huaf_t\right]        \\
	=&  \E\left[\alpha_t(z_t-v_t)^\top (g(z_t) - G(z_t,\xi_{t+\tau}))| \huaf_t \right]  \\
	=&  (z_t-v_t)^\top \left[\E  (g(z_t) - G(z_t,\xi_{t+\tau})) | \huaf_t \right]   \\      
	=&     (z_t-v_t)^\top \int   G(z_t,\xi) (\pi(\xi) -p_{[t]}^{t+\tau}(\xi) )d\xi  \\
	\le & 2DL_{1}  \int\vert \pi(\xi) -p_{[t]}^{t+\tau}(\xi) \vert d\xi 
	\end{align*}
	
	Taking expectation to the above inequation then summarization $ t $ from 1 to $ T-\tau $ we can get
	\begin{align*}
	\E\left[  (z_t-v_t)^\top \delta'_t |\huaf_t \right] &\le  2DL_{1}  \int\vert \pi(\xi) -p_{[t]}^{t+\tau}(\xi) \vert d\xi             \\  
	\E\left[ \sum_{t=1}^{T-\tau}\alpha_t(z_t-v_t)^\top \delta'_t \right] &\le  2DL_{1}\sum_{t=1}^{T-\tau}\alpha_t\E\int\vert \pi(\xi) -p_{[t]}^{t+\tau}(\xi) \vert d\xi .
	\end{align*}
\end{proof}

\begin{proof}[\textbf{Proof of lemma \ref{s lemma norm bound}}]
	The proof is similar to the proof of Proposition 3 in \cite{liu2015finite}. 
	Notice that in the specific RL problem setting our convex-concave problem can be written as
	\begin{align*}
	\min_x\max_y\left( L(x,y) = \langle b-Ax,y \rangle - \frac{1}{2}\Vert y\Vert^2_M \right),
	\end{align*} 
	the stochastic gradient vector $ G(x,y,\xi_t) $ can be written as
	\begin{align*}
	G(x,y,\xi_t) &= \left [ 
	\begin{array}{c}
	-\hat{A}_t^\top y \\ 
	- (\hat{b}_t-\hat{A}_t x -\hat{M}_ty)
	\end{array}
	\right ]  .\\                        
	\end{align*}
	Similar to the Lemma 2 in \cite{liu2015finite}, by using Assumption \ref{s assrl bound} and the definition of $ \hat{b},\hat{A}_t ,\hat{M}_t $ we can see, for $\forall \xi_t  $,$ \Vert G(x,y,\xi_t)\Vert^2  = 2 \cdot (\Vert \hat{A}_t^\top y \Vert^2 + \Vert \hat{b}_t-\hat{A}_t x -\hat{M}_ty \Vert^2 ) \le 2\cdot( \Vert A \Vert ^2 D^2 + (\Vert b \Vert   + (\Vert A \Vert + \lambda_M )D )^2 )\le  2\cdot(   2\Vert A \Vert D  +  \Vert b \Vert   +  \lambda_M  D   )^2  $.
	
	So Lipschitz constant $ L_{1} $ can be set to the upper bound of the gradient.
	
	The smooth constant $ L_{2} $ can be set similarly. $2\cdot (\Vert \frac{\partial G(x,y,\xi_t)}{\partial x}\Vert^2 +  \Vert \frac{\partial G(x,y,\xi_t)}{\partial y}\Vert^2 )\le 2\cdot (   ( \Vert A \Vert + \lambda_M)^2 +\Vert A\Vert^2)  \le 2\cdot(2\Vert A \Vert +\lambda_M)^2 $.
	
	
\end{proof}

		\newpage

	\bibliography{nips_2017}

\begin{thebibliography}{10}

\bibitem{bahdanau2016actor}
Dzmitry Bahdanau, Philemon Brakel, Kelvin Xu, Anirudh Goyal, Ryan Lowe, Joelle
  Pineau, Aaron Courville, and Yoshua Bengio.
\newblock An actor-critic algorithm for sequence prediction.
\newblock {\em arXiv preprint arXiv:1607.07086}, 2016.

\bibitem{bhatnagar2009convergent}
Shalabh Bhatnagar, Doina Precup, David Silver, Richard~S Sutton, Hamid~R Maei,
  and Csaba Szepesv{\'a}ri.
\newblock Convergent temporal-difference learning with arbitrary smooth
  function approximation.
\newblock In {\em Advances in Neural Information Processing Systems}, pages
  1204--1212, 2009.

\bibitem{dann2014policy}
Christoph Dann, Gerhard Neumann, and Jan Peters.
\newblock Policy evaluation with temporal differences: a survey and comparison.
\newblock {\em Journal of Machine Learning Research}, 15(1):809--883, 2014.

\bibitem{duchi2012ergodic}
John~C Duchi, Alekh Agarwal, Mikael Johansson, and Michael~I Jordan.
\newblock Ergodic mirror descent.
\newblock {\em SIAM Journal on Optimization}, 22(4):1549--1578, 2012.

\bibitem{durrett2016poisson}
Richard Durrett.
\newblock Poisson processes.
\newblock In {\em Essentials of Stochastic Processes}, pages 95--124. Springer,
  2016.

\bibitem{kober2013reinforcement}
Jens Kober, J~Andrew Bagnell, and Jan Peters.
\newblock Reinforcement learning in robotics: A survey.
\newblock {\em The International Journal of Robotics Research},
  32(11):1238--1274, 2013.

\bibitem{lazaric2012finite}
Alessandro Lazaric, Mohammad Ghavamzadeh, and R{\'e}mi Munos.
\newblock Finite-sample analysis of least-squares policy iteration.
\newblock {\em Journal of Machine Learning Research}, 13(1):3041--3074, 2012.

\bibitem{levin2009markov}
David~Asher Levin, Yuval Peres, and Elizabeth~Lee Wilmer.
\newblock {\em Markov chains and mixing times}.
\newblock American Mathematical Soc., 2009.

\bibitem{lin1993reinforcement}
Long-Ji Lin.
\newblock {\em Reinforcement learning for robots using neural networks}.
\newblock PhD thesis, Fujitsu Laboratories Ltd, 1993.

\bibitem{liu2015finite}
Bo~Liu, Ji~Liu, Mohammad Ghavamzadeh, Sridhar Mahadevan, and Marek Petrik.
\newblock Finite-sample analysis of proximal gradient td algorithms.
\newblock In {\em UAI}, pages 504--513. Citeseer, 2015.

\bibitem{maei2011gradient}
Hamid~Reza Maei.
\newblock {\em Gradient temporal-difference learning algorithms}.
\newblock PhD thesis, University of Alberta, 2011.

\bibitem{meyn2012markov}
Sean~P Meyn and Richard~L Tweedie.
\newblock {\em Markov chains and stochastic stability}.
\newblock Springer Science \& Business Media, 2012.

\bibitem{mnih2015human}
Volodymyr Mnih, Koray Kavukcuoglu, David Silver, Andrei~A Rusu, Joel Veness,
  Marc~G Bellemare, Alex Graves, Martin Riedmiller, Andreas~K Fidjeland, Georg
  Ostrovski, et~al.
\newblock Human-level control through deep reinforcement learning.
\newblock {\em Nature}, 518(7540):529--533, 2015.

\bibitem{nemirovski2004prox}
Arkadi Nemirovski.
\newblock Prox-method with rate of convergence o (1/t) for variational
  inequalities with lipschitz continuous monotone operators and smooth
  convex-concave saddle point problems.
\newblock {\em SIAM Journal on Optimization}, 15(1):229--251, 2004.

\bibitem{nemirovski2009robust}
Arkadi Nemirovski, Anatoli Juditsky, Guanghui Lan, and Alexander Shapiro.
\newblock Robust stochastic approximation approach to stochastic programming.
\newblock {\em SIAM Journal on Optimization}, 19(4):1574--1609, 2009.

\bibitem{silver2016mastering}
David Silver, Aja Huang, Chris~J Maddison, Arthur Guez, Laurent Sifre, George
  Van Den~Driessche, Julian Schrittwieser, Ioannis Antonoglou, Veda
  Panneershelvam, Marc Lanctot, et~al.
\newblock Mastering the game of go with deep neural networks and tree search.
\newblock {\em Nature}, 529(7587):484--489, 2016.

\bibitem{silver2014deterministic}
David Silver, Guy Lever, Nicolas Heess, Thomas Degris, Daan Wierstra, and
  Martin Riedmiller.
\newblock Deterministic policy gradient algorithms.
\newblock In {\em Proceedings of the 31st International Conference on Machine
  Learning}, pages 387--395, 2014.

\bibitem{sutton1988learning}
Richard~S Sutton.
\newblock Learning to predict by the methods of temporal differences.
\newblock {\em Machine learning}, 3(1):9--44, 1988.

\bibitem{sutton1998reinforcement}
Richard~S Sutton and Andrew~G Barto.
\newblock {\em Reinforcement learning: An introduction}.
\newblock MIT press Cambridge, 1998.

\bibitem{sutton2009convergent}
Richard~S Sutton, Hamid~R Maei, and Csaba Szepesv{\'a}ri.
\newblock A convergent $ o (n) $ temporal-difference algorithm for off-policy
  learning with linear function approximation.
\newblock In {\em Advances in Neural Information Processing Systems}, pages
  1609--1616, 2009.

\bibitem{sutton2009fast}
Richard~S Sutton, Hamid~Reza Maei, Doina Precup, Shalabh Bhatnagar, David
  Silver, Csaba Szepesv{\'a}ri, and Eric Wiewiora.
\newblock Fast gradient-descent methods for temporal-difference learning with
  linear function approximation.
\newblock In {\em Proceedings of the 26th International Conference on Machine
  Learning}, pages 993--1000, 2009.

\bibitem{tagorti2015rate}
Manel Tagorti and Bruno Scherrer.
\newblock On the rate of convergence and error bounds for lstd ( $ lambda $).
\newblock In {\em Proceedings of the 32nd International Conference on Machine
  Learning}, pages 1521--1529, 2015.

\bibitem{tsitsiklis1997analysis}
John~N Tsitsiklis, Benjamin Van~Roy, et~al.
\newblock An analysis of temporal-difference learning with function
  approximation.
\newblock {\em IEEE transactions on automatic control}, 42(5):674--690, 1997.

\bibitem{yu2015convergence}
H~Yu.
\newblock On convergence of emphatic temporal-difference learning.
\newblock In {\em Proceedings of The 28th Conference on Learning Theory}, pages
  1724--1751, 2015.

\end{thebibliography}
	\bibliographystyle{plain}

\end{document}